\documentclass[acmlarge]{acmart}
\usepackage{lipsum} % Just for the demo, you can safely remove it.
\usepackage[tight,footnotesize]{subfigure}
\usepackage{enumitem}
\usepackage{float}
\usepackage{multicol,multienum}
\usepackage{breqn}
\usepackage{stfloats}
\usepackage{multirow}
\usepackage{diagbox}
\usepackage{bm}
\usepackage{graphicx}
\usepackage{tabularx}
\usepackage{algorithm}
\usepackage{algorithmic}
\newtheorem{theorem}{Theorem}
\newtheorem{assumption}{Assumption}
\usepackage[tight,footnotesize]{subfigure}
\usepackage{url}
\usepackage{doi}
\usepackage{color}
\AtBeginDocument{%
  \providecommand\BibTeX{{%
    \normalfont B\kern-0.5em{\scshape i\kern-0.25em b}\kern-0.8em\TeX}}}

\setcopyright{acmcopyright}
\copyrightyear{2020}
\acmYear{2020}
\acmDOI{xxxxxxxxx}

\acmJournal{TOIT}
\acmVolume{1}
\acmNumber{2}
\acmArticle{1}
\acmMonth{9}

\begin{document}

\title{Towards Communication-efficient and Attack-Resistant Federated Edge Learning for Industrial Internet of Things}
%Local Differential Privacy Empowered Asynchronous
\author{Yi Liu}
\orcid{0000-0002-0811-6150}
%\email{webmaster@marysville-ohio.com}
\affiliation{%
  \institution{Heilongjiang University and Tencent Jarvis Lab}
  \streetaddress{74 Xuefu Rd, Nangang District}
  \city{Harbin}
  \state{Heilongjiang}
  \country{China}
  \postcode{150080}}
  \email{97liuyi@ieee.org}

\author{Ruihui Zhao}
\affiliation{%
  \institution{Tencent Jarvis Lab}
  \streetaddress{10000 Shennan Avenue, Nanshan District}
  \city{Shenzhen}
  \country{China}}
\email{zacharyzhao@tencent.com}

\author{Jiawen Kang}
\affiliation{%
  \institution{Nanyang Technological University}
  \city{Singapore}
  \country{Singapore}
}
\email{kavinkang@ntu.edu.sg}

\author{Abdulsalam Yassine}
\affiliation{%
 \institution{Lakehead University}
 \streetaddress{P7B 5E1}
 \city{Ontario}
 \country{Canada}}
\email{ayassine@lakeheadu.ca}

\author{Dusit Niyato}
\affiliation{%
	\institution{Nanyang Technological University}
	\city{Singapore}
	\country{Singapore}
}
\email{dniyato@ntu.edu.sg}

\author{Jialiang Peng}
\authornote{Jialiang Peng is the corresponding author.}
\affiliation{%
  \institution{Heilongjiang University}
  \streetaddress{74 Xuefu Rd}
  \city{Harbin}
  \state{Heilongjiang}
  \country{China}}
\email{jialiangpeng@hlju.edu.cn}

\renewcommand{\shortauthors}{Yi Liu et al.}

\begin{abstract}
Federated Edge Learning (FEL) {allows} edge nodes to train a global deep learning model collaboratively for edge computing in the Industrial Internet of Things (IIoT), which significantly promotes the development of Industrial 4.0. However, FEL faces two critical challenges: communication overhead and {data privacy}. FEL suffers from expensive communication overhead when training {large-scale multi-node models}. Furthermore, due to the vulnerability of FEL to gradient leakage and label-flipping attacks, the training process of the global model is easily compromised by adversaries. To address these challenges, we propose a communication-efficient and privacy-enhanced asynchronous FEL framework for edge computing in IIoT. First, we introduce an asynchronous model update scheme to reduce the computation time that edge nodes wait for global model aggregation. Second, we propose an asynchronous local differential privacy mechanism, which improves communication efficiency and mitigates gradient leakage attacks by adding well-designed noise to the gradients of {edge nodes}. Third, we design a cloud-side malicious node detection mechanism to detect malicious nodes by testing the local model quality. Such a mechanism can avoid malicious nodes participating in training to mitigate label-flipping attacks. Extensive experimental studies on two real-world datasets demonstrate that the proposed framework can not only improve communication efficiency but also mitigate malicious attacks while its accuracy is comparable to traditional FEL frameworks.
\end{abstract}



\maketitle

\section{Introduction}
Advanced artificial intelligence technologies are widely used in various applications for edge computing in the Industrial Internet of Things (IIoT), such as smart transportation, smart home, and smart healthcare, which have greatly promoted the development of Industry 4.0 \cite{ref-1,ref-2,ref-74}. For example, Liu \textit{et al.} in \cite{ref-3} developed a new edge intelligence-based traffic flow prediction application. {Edge nodes use a large amount of user data collected and shared by nodes to train advanced deep learning models \cite{ref-75}. But edge nodes may leak private data during {the} data collection, data transmission, and data sharing {processes}, which {would} lead to serious privacy problems for both data owners and data providers \cite{ref-41,ref-42,ref-43,ref-44,ref-45}.} To solve these problems, previous {works} mainly {applied} encryption algorithms to encrypt the transmitted and shared data without destroying the privacy of the data owners. Yuan \textit{et al.} in \cite{ref-28} used the encryption algorithms to {propose} a privacy-preserving social discovery service architecture for {sharing} images in social media {websites}. However, it is {a challenge} for edge nodes to {efficiently} encrypt large amounts of data in IIoT {due to the high computational cost of  encryption algorithms} \cite{ref-46}. Therefore, we need to seek a communication-efficient and secure way for edge nodes to train advanced deep learning models.

To alleviate the privacy leakage problem of training data in IIoT, a promising edge learning paradigm called {Federated Edge Learning} (FEL) \cite{ref-5,ref-14} was proposed for edge nodes to perform collaborative deep learning in IIoT. {In contrast to centralized distributed learning methods with a central data server, FEL allows edge nodes to keep the training dataset locally, perform local model training, and only send shared gradients to the cloud instead of raw data \cite{ref-3,ref-5,ref-12,ref-46,ref-53}.} Therefore, the current {studies} mainly {focus} on using FEL framework to achieve secure collaborative learning between edge nodes in IIoT. {For example, Lu \textit{et al.} in \cite{ref-29,ref-20,ref-80} proposed a blockchain-based FEL scheme to achieve secure data sharing, thereby achieving secure collaborative learning in IIoT.}

However, there are {still} two challenges in the current {studies}: communication overhead and {data} privacy \cite{ref-76}. The former indicates that communication bandwidth and network delay {become} the bottlenecks of {FEL further development}. {{The reason is that the exchange of a large number of gradients between the cloud and the edge nodes will cause expensive communication overhead, thereby affecting the model training process} \cite{ref-5}.} The latter indicates that the FEL is vulnerable to gradient leakage \cite{ref-15,ref-40} and poisoning attacks \cite{ref-2,ref-10,ref-11,ref-12,ref-16,ref-27}. The {untrusted} cloud {can acquire} the private data of edge nodes by launching gradient leakage attacks. {For example,} Zhu \textit{et al.} in \cite{ref-24} utilized the gradients uploaded by edge nodes to restore local training data. Furthermore, the edge nodes perform local model training by using local data, but this provides opportunities for malicious nodes to {carry out} label-flipping attacks (i.e., a type of poisoning attacks). Fung \textit{et al.} in \cite{ref-10} demonstrated that a malicious node can modify the {labels} of the local training dataset to {execute} a label-flipping attack. {In this attack, the edge node learns the wrong knowledge from the data with the flipped label and generates the wrong model update, which leads to the cloud aggregation of the wrong model update. Therefore, this will degrade the performance of the system.}

Differential Privacy (DP) techniques \cite{ref-33} are often used to mitigate gradient leakage attacks. Specifically, DP noise mechanism (e.g., Laplace noise mechanism \cite{ref-47} and Gaussian noise mechanism \cite{ref-48}) can be {used to prevent gradient leakage attacks by adding noises to the gradients.} {A differentially private FEL framework was proposed in \cite{ref-6,ref-20}} to mitigate gradient leakage attacks for IIoT applications. However, the DP mechanism only provides {the} privacy guarantee in the data analysis phase {rather than} the node-level privacy. {Zhao \textit{et al.} in \cite{ref-2} presents client-side detection schemes to detect malicious client for label-flipping attacks.} Such {a scheme} can prevent malicious nodes from participating in model training. However, {the above-mentioned} schemes {actually} send the sub-model generated by an edge node to other nodes for detection, {which may result in the information leakage} of the sub-models.

To develop a {communication-efficient and secure} {FEL} framework for edge computing in IIoT, {it needs to} improve {the} communication efficiency and mitigate malicious attacks without reducing {the accuracy of FEL}. Specifically, we propose {an} asynchronous model update scheme to improve {the} communication efficiency. Furthermore, we design {the} asynchronous local differential privacy (ALDP) and malicious node detection {mechanisms}. The ALDP mechanism can protect {the} node-level privacy, which is the {main} difference from the {existing} DP {mechanisms}. Unlike the traditional client-side detection scheme, we propose {the} cloud-side detection scheme to avoid collusion and attack between nodes. {The cloud-only needs to {use} a testing dataset to test the quality of each local model to {perform the} malicious node detection. Based on the test results, the cloud aggregates model updates uploaded by the normal nodes to obtain the global model. We evaluate the proposed framework on two real-world datasets MNIST \cite{ref-50} and CIFAR-10 \cite{ref-51}.} The experimental results show that the proposed framework can achieve high-efficiency communication and mitigate inference and label-flipping attacks without loss of accuracy. The {main} contributions of this paper are summarized as follows:
\begin{itemize}
%	\item We propose an asynchronous model update scheme to develop a  communication-efficient and secure FL framework for edge computing in IIoT.
	\item We develop a communication-efficient asynchronous federated edge learning architecture by leveraging an asynchronous model update scheme instead of the synchronous update scheme.
	
	\item We achieve the node-level privacy guarantee of updated models in asynchronous federated edge learning by incorporating local differential privacy mechanism into an asynchronous model update scheme. We also mathematically prove that this mechanism can achieve convergence.
	
	\item We propose a cloud-side malicious node detection mechanism to evaluate the model quality for mitigating label-flipping attacks in the proposed framework.
	
	\item We conduct extensive experiments on two real-world datasets to demonstrate the performance of the proposed framework. Experimental results show that the framework can mitigate malicious attacks and achieve efficient communication without compromising accuracy.
\end{itemize}
The rest of this paper is organized as follows. {Section \ref{sec-2} is the Related Work part, which reviews the literature on communication-efficient and attack-resistant federated edge learning. Section \ref{sec-3} is the Preliminary part which presents the background knowledge of the local differential privacy, federated edge learning, and malicious attacks. Section \ref{sec-4} is the System Model part which presents our proposed model and design goals. In Section \ref{sec-5}, we present the Local Differentially Private Asynchronous Federated Edge Learning Framework. {Section \ref{sec-6} discusses} the experimental results. Concluding remarks are described in Section \ref{sec-7}.}

\section{Related Work}\label{sec-2}
{The convergence of Artificial Intelligence and Edge Computing has spawned many new paradigms of Edge Intelligence (Edge-AI), especially the Federated Edge Learning (FEL) paradigm.} Although the Federated Edge Learning Framework has been applied in many fields thanks to its privacy protection features, there are still two challenges: communication overhead and malicious attacks. The state-of-the-art literature related to the proposed framework presented in this paper can be divided into two main areas. The first area involves how to improve communication efficiently. The second area considers how to resist malicious attacks.

\subsection{Communication-efficient Federated Egde Learning}
{To improve the communication efficiency of the FEL framework, researchers generally use gradient quantization, gradient sparsification and asynchronous update techniques. These methods mainly reduce the number of gradients in the interaction between the cloud and the edge nodes in the framework to achieve efficient communication.}

\textbf{Gradient Quantization:} Gradient quantization generally quantizes 32-bit floating-point gradients to 8 bits, 4 bits, or even 2 bits to reduce communication time \cite{ref-9}. For example, Alistarh \textit{et al.} in \cite{ref-8} proposed a communication-efficient stochastic gradient descent (QSGD) algorithm to greatly improve the communication efficiency of the distributed learning framework. Specifically, the authors used gradient quantization and gradient encoding techniques to losslessly compress the gradients uploaded by edge nodes. Samuel \textit{et al.} in \cite{ref-55} proposed a variance reduction-based gradient quantization algorithm which can realize a distributed learning framework with lossless precision and communication efficiency. Although gradient quantization can greatly reduce the size of the gradient, it will lose a large amount of rich semantic information in the gradient, which makes it difficult to ensure that the performance of the framework is not affected by information loss.

\textbf{Gradient Sparsification:} Edge nodes only upload gradients with larger values to the cloud, and zero gradients with other smaller values to achieve efficient distributed learning in communication \cite{ref-6}. Previous works \cite{ref-56,ref-57,ref-58,ref-59,ref-60,ref-61} have proposed many sparsification methods and also provided a strict mathematical proof of convergence. Lin \textit{et al.} in \cite{ref-6} proposed a high-precision distributed learning framework that improves communication efficiency by 300 $ \times $. Specifically, the authors first used the gradient sparsification technique to sparse the gradient, then upload the compressed and sparse gradient to the cloud, and the cloud decompresses and aggregates them.

{\textbf{Asynchronous Update:} In a federated learning system, edge nodes can interact with the server regularly for model updates through synchronous or asynchronous updates \cite{ref-19}. In particular, the previous works \cite{ref-19,ref-20,ref-77,ref-78,ref-79,ref-80} used asynchronous updates to improve the communication efficiency of the system. Specifically, the asynchronous update method can reduce the calculation time of the node by reducing the calculation waiting time of the edge node, thereby improving the communication efficiency of the system. The reason is that in the synchronous federated learning system, each node needs to wait for the calculation of other nodes to perform update aggregation, which greatly reduces the efficiency of communication.}

Although the above the-state-of-art communication efficient methods have achieved great success, none of these methods can resist malicious attacks initiated by nodes or clouds. The reason is that these methods assume that the cloud or node is honest. Furthermore, these methods are difficult to converge with attack tolerance techniques. Therefore, we need to design a solution that can converge well with attack tolerance techniques.

\subsection{Attack-Resistant Federated Edge Learning}
A common setting of the federated edge learning framework includes cloud and edge nodes. In FEL framework, the edge nodes and the cloud have the same power and have the same influence on the framework. So both malicious nodes and malicious clouds can launch malicious attacks on the framework and hurt the performance of the framework. Many researchers have proposed many attack-resistant methods to defend against malicious attacks, but these methods can be divided into three categories: Byzantine robust aggregation, anomaly detection, and differential privacy.

\textbf{Byzantine Robust Aggregation:} We assume that there are some colluding Byzantine edge nodes in FEL framework. These Byzantine nodes will send some wrong messages to the cloud, such as Gaussian noise, zero gradient attack, arbitrary gradient attack, and other messages harmful to the performance of the framework. Previous works \cite{ref-62,ref-63,ref-64,ref-65,ref-66,ref-67} designed many Byzantine aggregation algorithms based on different aggregation rules, e.g., median, geometric mean, arithmetic mean, etc. However, the sensitivity of these Byzantine robust aggregation algorithms will be affected by the gradient variance in stochastic gradient descent \cite{ref-65}. 

\textbf{Anomaly Detection:} In malicious attacks, we can treat model updates uploaded by malicious nodes as anomalies, so that anomaly detection techniques can be used to defend against malicious attacks \cite{ref-70}. For example, Li \textit{et al.} in \cite{ref-71} used dimension reduction techniques to detect abnormal model updates. Specifically, the authors used a binary classification technique to distinguish between benign updates and malicious updates after dimension reduction on model updates. Yang \textit{et al.} in \cite{ref-72} proposed a neural network-based method, i.e., FedXGBoost model to detect the malicious model updates. However, these methods are impractical in the actual operation of the FEL framework. The reason is that these models themselves are more complex, which will increase the complexity and operating costs of the FEL framework.

\textbf{Differential Privacy:} Differential Privacy (DP) is a privacy protection technology that disturbs the parameters by adding noise (such as Gaussian noise mechanism, Laplacian noise mechanism) to the parameters of the deep learning model so that the attacker cannot know the real parameters \cite{ref-33,ref-40}. Therefore, Geyer \textit{et al.} in \cite{ref-73} added DP noise to the weights so that protecting the privacy of the federated deep learning model by using the Gaussian noise mechanism. However, DP noise can only protect the privacy of the aggregated parameters (i.e., the privacy of the cloud) but cannot protect the privacy of nodes. 

\section{Preliminary}\label{sec-3}
{In this section, we {briefly introduce} {the} background knowledge {related to} local differential privacy, federated edge learning, gradient leakage attacks, and poisoning attacks.}

\subsection{Local Differential Privacy}
Local Differential Privacy (LDP) \cite{ref-48} technique was proposed to collect sensitive data without making any assumptions on a trusted cloud {environment}. {Unlike DP techniques, LDP protects the privacy of user data by using local disturbances \cite{ref-33}. Specifically, LDP mainly focuses on protecting the privacy of a single node or a single user rather than the privacy of all nodes or all users.} In this context, the cloud can collect user data without compromising privacy \cite{ref-36}. Let  $x$ be the input {vector}, and a perturbed vector $v^*$ for $x$ is output through a {randomized} algorithm $\mathcal{M}$. Therefore, $(\varepsilon ,\delta )$-LDP can be defined as follows:
\begin{definition}\label{de-1}
{\textbf{(\textit{$\bm{(\varepsilon ,\delta )}$-LDP})}\cite{ref-new-peng-1}. \textit{A randomized algorithm $\mathcal{M}:{\mathbb{X}} \to {\mathbb{V}}$ with domain $\mathbb{X}$ and range $\mathbb{V} \subseteq \mathbb{X}$ satisfies $(\varepsilon ,\delta )$-LDP if and only if for any two inputs $x,{x'} \in \mathbb{X}$ and output ${v^*} \in \mathbb{V}$:
\begin{equation}\label{eq-1}
\Pr [\mathcal{M}(x) = v^*] \le {e^\varepsilon } \cdot \Pr [\mathcal{M}(x') = {v^*}] + \delta,
\end{equation}
where $\Pr [ \cdot ]$ is the conditional probability  {density} function that depends on $\mathcal{M}$, $\varepsilon$ represents privacy budget that controls the trade-off between privacy and utility, and $\delta$ is a sufficiently small positive real number.
It also implies a higher privacy budget $\varepsilon$ means a lower privacy protection. The above definition can be changed to $\varepsilon$-LDP when $\delta = 0$.}}
\end{definition}

According to Definition {\ref{de-1}}, {randomized} perturbations {should be} performed by users rather than the cloud. In the context of {generic} {settings} for {FEL}, the cloud only aggregates and analyzes the results after local disturbances {that controlled by the privacy budget $\varepsilon$}. {It} ensures that the cloud cannot distinguish $x$ or $x'$ from  $v^*$ with high-confidence.

\subsection{Federated Edge Learning}
The centralized machine learning model uses distributed learning algorithms such as distributed stochastic gradient descent (D-SGD) \cite{ref-52} to analyze and model aggregated data in the cloud. {However, the centralized machine learning model requires data sharing or data aggregation to update the cloud model, which will leak the privacy of user data. To address these challenges, Google proposed the federated edge learning (FEL) framework \cite{ref-5,ref-14} to train a shared global model by using federated SGD (FedSGD) or federated averaging (FedAVG) \cite{ref-5} algorithm while using the local dataset without sharing raw data \cite{ref-3}.} In each iteration $t$ ($t \in \{ 1,2, \ldots ,T\}$), the procedure of FEL {can be described} as follows:
\begin{enumerate}[label=\roman*)]
\item \textbf{\textit{{Phase 1, Initialization:}}} All nodes that want to participate in FEL training check in to the cloud. Then the cloud selects nodes with good network status or strong computing capabilities to participate in the training task. After that, the cloud sends a pre-trained global model ${\omega _{t}}$ to each selected node.
\item \textbf{\textit{{Phase 2, Training:}}} Each node trains global model $\omega _t^k \leftarrow {\omega _t}$ by using local dataset to obtain the updated global model $\omega _t^{k + 1}$ in {each iteration}. In particular, for the $k$-th edge node {($k \in \{ 1,2, \ldots ,K\}$)}, the loss function needs to be optimized as follows:
 {\begin{equation}
\arg \mathop {\min }\limits_{\omega  \in \mathbb{R}}{F_k}(\omega ),{F_k}(\omega ) = \frac{1}{{{D_k}}}\sum\nolimits_{i \in {D_k}} {{f_i}} (\omega ),
\end{equation}}

{where $D_k$ denotes the size of local dataset that contains input-output {vector} pairs $(x_i, y_i)$, ${x_i},{y_i} \in \mathbb{R}$, $\omega$ is local model parameter, and ${f_i}(\omega )$ is a local loss function (e.g., ${f_i}(\omega ) = \frac{1}{2}({x_i}^T\omega  - {y_i})$).}
\item \textbf{\textit{Phase 3, Aggregation:}} {The cloud uses} the FedSGD or FedAVG algorithm to obtain a new global model $\omega_{t + 1}$ for the next iteration, i.e.,
\begin{equation}
{\omega _{t + 1}} \leftarrow {\omega _t} - \frac{1}{K}\sum\limits_{k = 1}^K {{F_k}(\omega )},
\end{equation}
{where $K$ denotes the number of edge nodes.}
Then {again,} the cloud sends the new global model $\omega_{t + 1}$ to each edge node.
\end{enumerate}

The above steps are repeated until the global model of the cloud converges. This learning method avoids data sharing and collaboratively trains a shared global model, thereby realizing privacy protection \cite{ref-4}.
%Both the nodes and the cloud repeats the above process until the global model reaches convergence. This paradigm significantly reduces the risks of privacy leakage by decoupling the model training from direct access to the raw training data \cite{ref-4}.

Generally, we can divide the model update {methods} in FEL into two categories: {synchronous model update \cite{ref-5} and asynchronous model update \cite{ref-37} methods}, as illustrated in Fig. \ref{fig-10}. {For} the synchronous model update, all edge nodes simultaneously upload the model {updates} to the cloud after performing a round of local model training. For the asynchronous model update, in contrast, {each edge node performs local model training and then uploads model updates asynchronously, i.e., without waiting for other nodes} to the cloud. 

\begin{figure}[!t]
	\centering
	\includegraphics[width=1\linewidth]{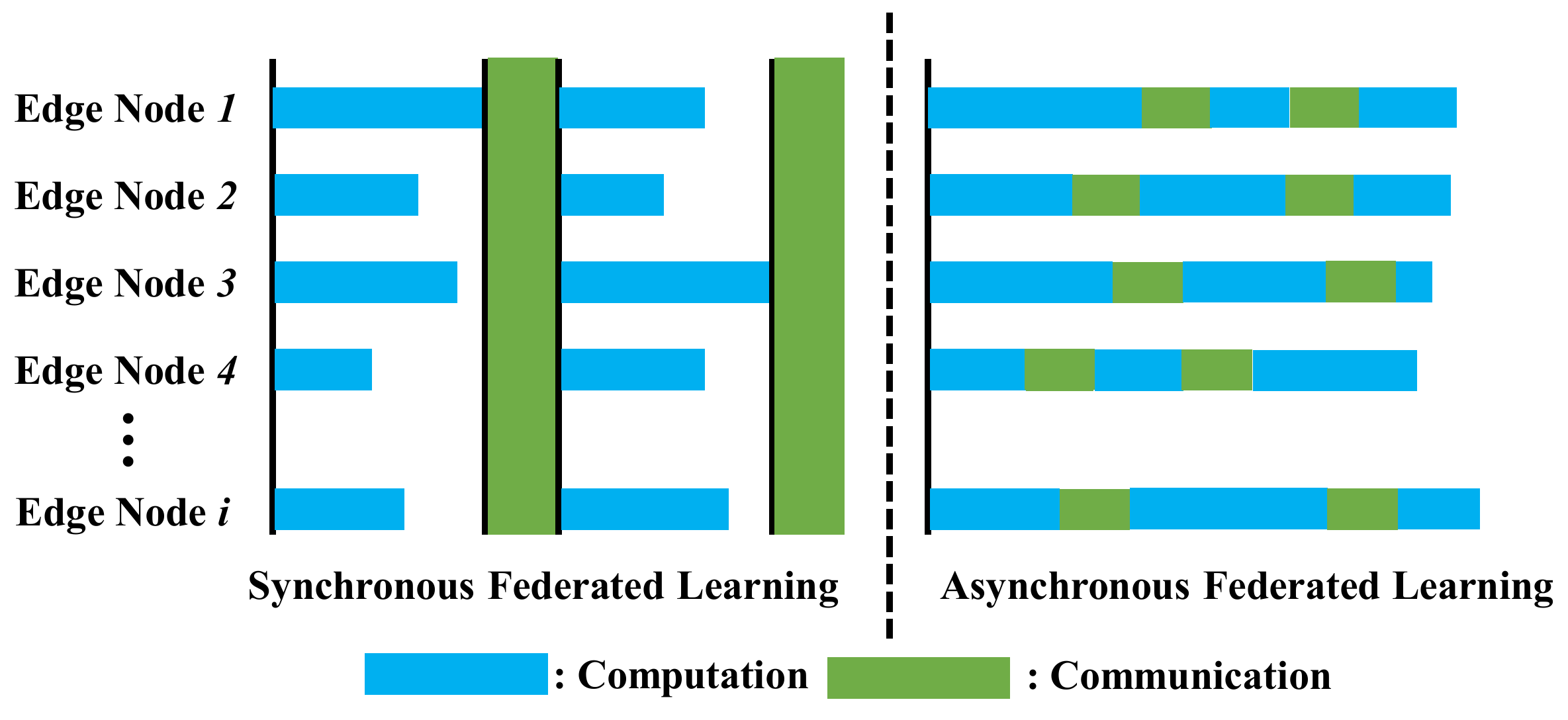}
	\caption{The overview of the synchronous model update and asynchronous model update schemes.}
	\label{fig-10}
\end{figure}
\begin{figure}[t]
	\centering
	\large
	\subfigure []{\includegraphics[width=0.45\linewidth]{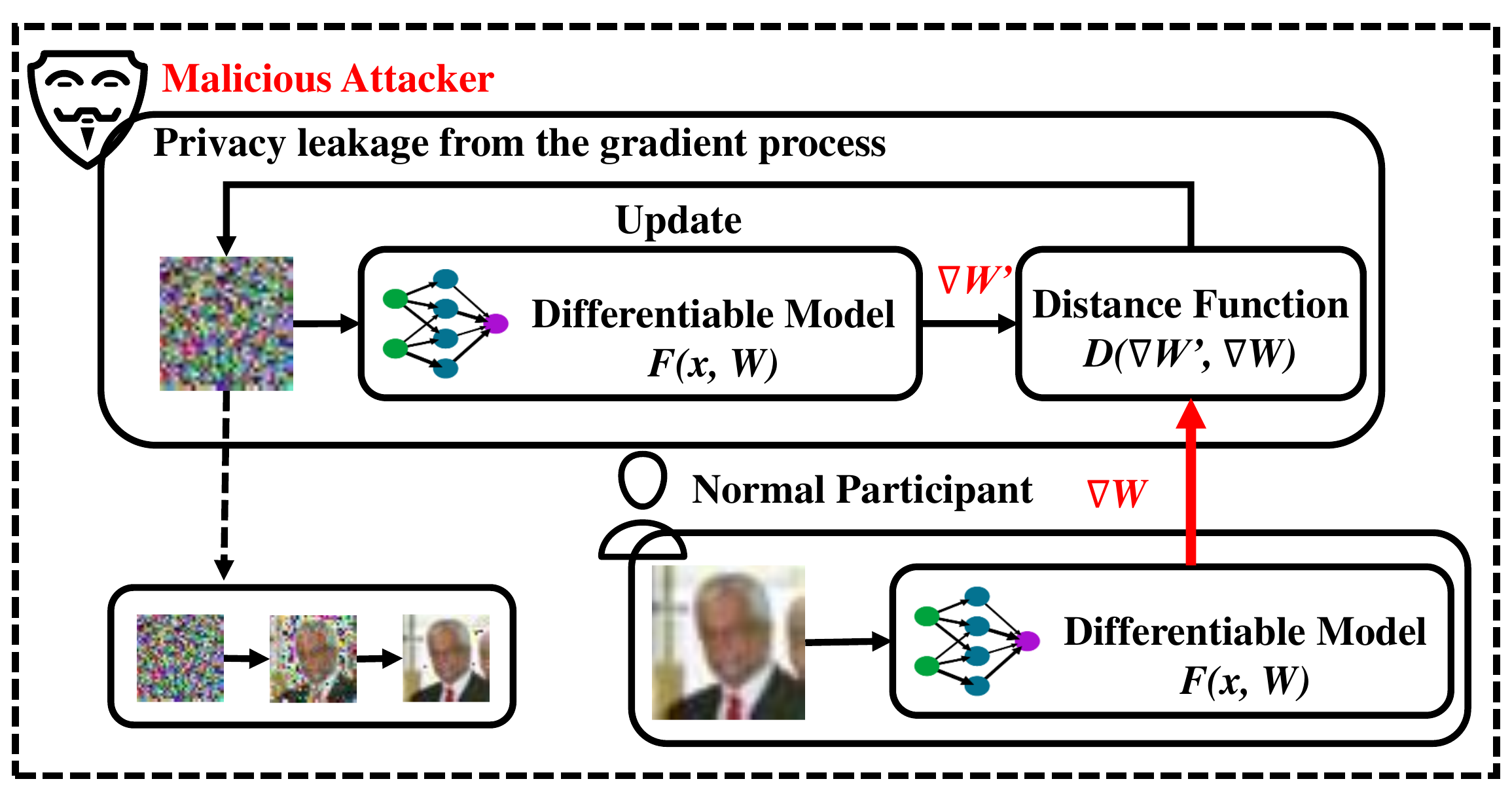}
		\label{11-a-1}}
	\hfill
	\subfigure[]{	\includegraphics[width=0.45\linewidth]{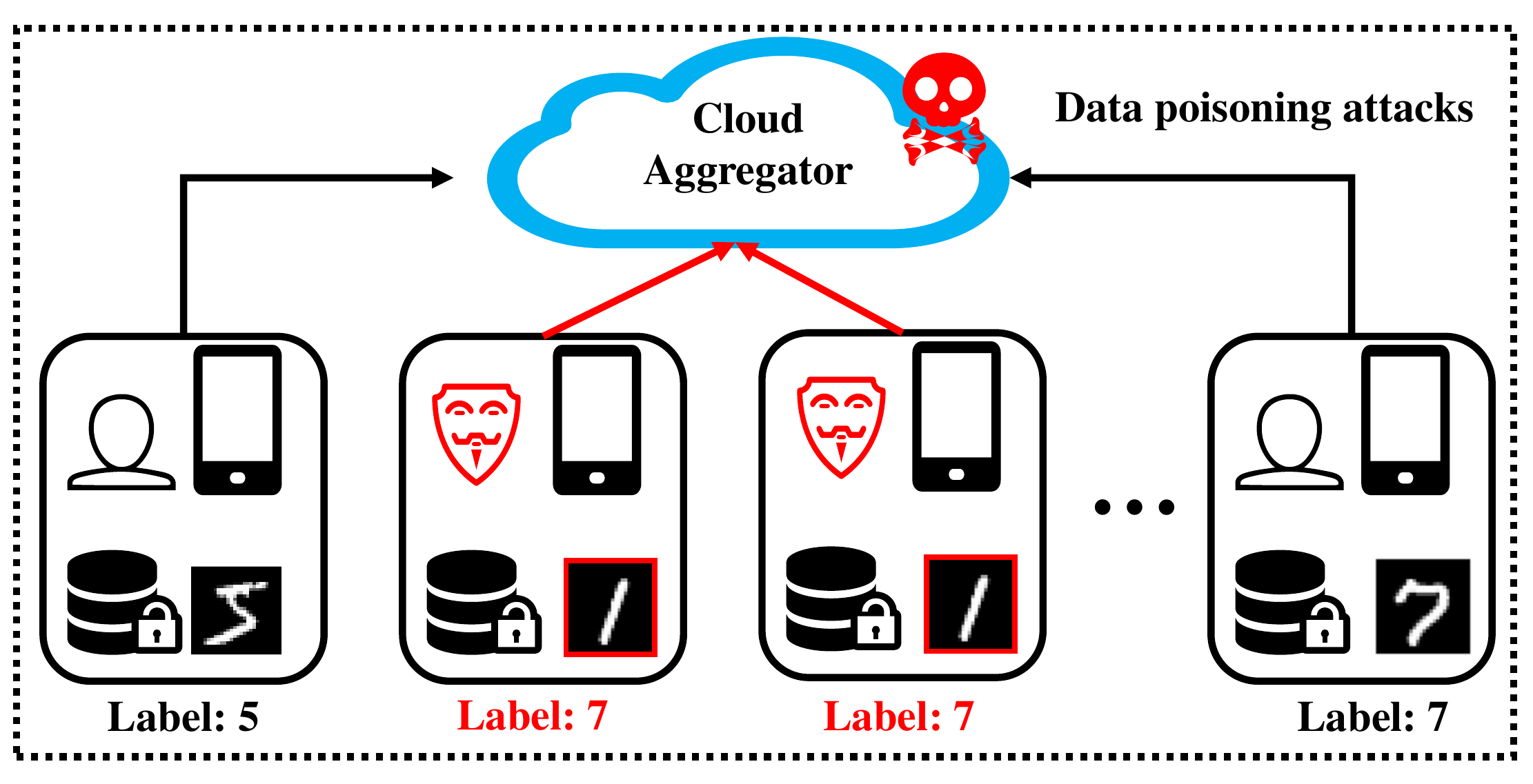}
		\label{11-b-1}}
	\caption{(a) Privacy leakage from the gradient process \cite{ref-24}. (b) Data poisoning attacks on the federated edge learning \cite{ref-10}.}
	\label{fig-11}
\end{figure}

\subsection{{Gradient Leakage Attacks and Data Poisoning Attacks}}
Assumed that there are several edge nodes and a cloud may be malicious in the edge computing environment, we further present possible attacks to FEL. 

\textbf{Gradient Leakage Attack:} The malicious cloud can launch gradient leakage attacks to use the gradient information uploaded by edge nodes to recover the sensitive private training data \cite{ref-24}, as illustrated in Fig. \ref{11-a-1}. Zhu \textit{et al.} in \cite{ref-24} utilized gradient matching algorithm to recover nodes' private training data. Specifically, the authors randomly constructed a pair of noise data pairs $(X', Y')$ to reconstruct the node's private data pair $(X, Y)$. We assume that the attacker knows the model structure and model parameters of the node, which is easy to achieve in FEL setting. The attacker's goal is to make the gradient generated by the noise data pair infinitely close to the gradient generated by the real private data pair. The formal definition of this goal is as follows:
\begin{equation}
\begin{array}{l}
\mathrm{Constructing}\;(X',Y'),\\
s.t.\frac{{\partial \mathcal{L}(\mathcal{F}(W,X');Y')}}{{\partial W}} = \frac{{\partial \mathcal{L}(\mathcal{F}(W,X);Y)}}{{\partial W}}.
\end{array}
\end{equation}
Since the gradients always contain the original data distribution information of the training samples, the sensitive training samples would be stolen by the malicious cloud or malicious node once the gradients are leaked.

\textbf{Data Poisoning Attack:} As shown in Fig. \ref{11-b-1}, the malicious edge nodes could launch data poisoning attacks to poison the global model in FEL by changing the behavior of trained models on specific inputs such as changing the label of training data samples and injecting toxic training data samples. Specifically, a typical type of data poisoning attacks is label-flipping attack \cite{ref-10}. Label-flipping attacks could be executed to poison the global model in FEL by changing the behaviors of trained models on specific inputs such as changing the data labels of training samples and injecting toxic training data samples. 

{Previous studies have} shown that FEL frameworks in edge computing are also vulnerable to gradient leakage attacks and label-flipping attacks \cite{ref-11,ref-12}. {Adversaries} can launch these attacks to acquire sensitive data from nodes and reduce the accuracy of FEL model and even cause FEL model to fail to converge. 

\section{System Model}\label{sec-4}
We consider {a general} setting for {FEL} in edge computing, where a cloud and $K$ edge nodes work collaboratively to train a global model with {a given} training algorithm (e.g., convolutional neural network (CNN)) for a specific task (e.g., classification task). The edge nodes train the model locally by using the local dataset and upload their updated gradients to the cloud. The cloud uses the FedSGD algorithm or other aggregation algorithms to aggregate these gradients to obtain the global model. In the end, the edge node will receive the {updated} global model sent by the cloud for its own use.

\subsection{Treat Model}
We focus on a collaborative model learning in edge computing to accomplish a classification or prediction task. In this context, FEL is vulnerable to two types of threats: {gradient leakage attacks} and {label-flipping attacks}. {For gradient leakage attacks, we assume that the cloud is untrusted \cite{ref-14}. The malicious cloud may reconstruct the private data of nodes by launching gradient leakage attacks \cite{ref-15}. {The reason is that the cloud can use the rich semantic information of the gradients uploaded by the edge nodes to infer the privacy information of the edge nodes.} For label-flipping attacks, it is assumed that there is a proportion $p$ ($0 < p < 1$) of malicious nodes among all the edge nodes. The malicious nodes damage the training process of the global model by flipping the labels of the local datasets and uploading error or low-quality model updates.} 

{\textbf{Remark:} Since the above-mentioned two attacks do not need to have prior knowledge about the model training process, it is easy for an adversary to carry out \cite{ref-2,ref-10,ref-11,ref-12,ref-16,ref-24,ref-53}. On the other hand, in a federated learning system, edge nodes have complete autonomous control over their own data sets and model updates, which provides malicious nodes with opportunities to maliciously tamper with data labels and model updates.}

\subsection{Our Proposed Architecture}
This work considers an asynchronous FEL that involves multiple edge nodes for collaborative model learning in IIoT, as illustrated in Fig. \ref{fig-1}. This framework includes the cloud and edge nodes. Furthermore, the proposed framework includes two secure mechanisms: a local differential privacy and a malicious node detection mechanism. The functions of the participants are described as follows:
\begin{itemize}
	\item \textbf{\textit{Cloud:}} {The cloud} is generally a cloud server with rich computing resources. The cloud {contains} three functions: (1) initializes the global model and sends the global model to the all edge nodes; (2) aggregates the gradients uploaded by the edge nodes until the model converges; (3) performs the malicious node detection mechanism.
	\item \textbf{\textit{Egde Node:}} {Edge nodes are generally agents and clients, such as mobile phones, personal computers, and vehicles, which contain buffers and functional mechanisms. The edge node uses the local dataset to train local models and uploads the gradients to the cloud until the global model converges. {The buffer is deployed in the edge node that contains the scheduler and coordinator. In this paper, we use the buffers to store the local accumulated gradients of the edge nodes and the scheduler in the cloud to achieve the asynchronous model update. Note that buffer is the storage area responsible for storing gradients in the FL system.} In particular, the coordinator in the edge node is responsible for triggering nodes to perform local model training.}
\end{itemize}	
The functions of the proposed mechanisms are {further} described as follows:
\begin{itemize}
	\item \textbf{\textit{Local Differential Privacy Mechanism:}} The local differebtial privacy mechanism is deployed in the edge nodes, which can perturb the {raw} local gradients of the nodes to ensure that the cloud cannot distinguish input gradients from edge nodes with high-confidence, thereby protecting the privacy of the node.
	\item \textbf{\textit{Malicious Node Detection Mechanism:}} The malicious detection mechanism is deployed in the cloud as illustrated in Fig. \ref{fig-1}, which can detect malicious edge nodes to improve the security of the framework.
\end{itemize}

\subsection{Design Goals}
In this paper, our goal is to develop a communication-efficient and attack resistance asynchronous {FEL} framework for edge intelligence in IIoT. Similar to most FEL frameworks, the proposed framework {achieves} secure model training for edge computing in IIoT. First, the proposed framework can significantly improve communication efficiency by using an asynchronous model update method. Second, the proposed framework {can} mitigate gradient leakage and label-flipping attacks. This framework {applies} a local differential privacy mechanism to mitigate gradient leakage attacks. Moreover, the proposed framework {introduces} a malicious node detection mechanism to prevent label-flipping attacks. Third, the performance of the proposed framework is comparable to {that of} traditional FEL frameworks.

\section{Local Differentially Private Asynchronous Federated Edge Learning Framework}\label{sec-5}

\begin{figure}[!t]
	\centering
	\includegraphics[width=1\linewidth]{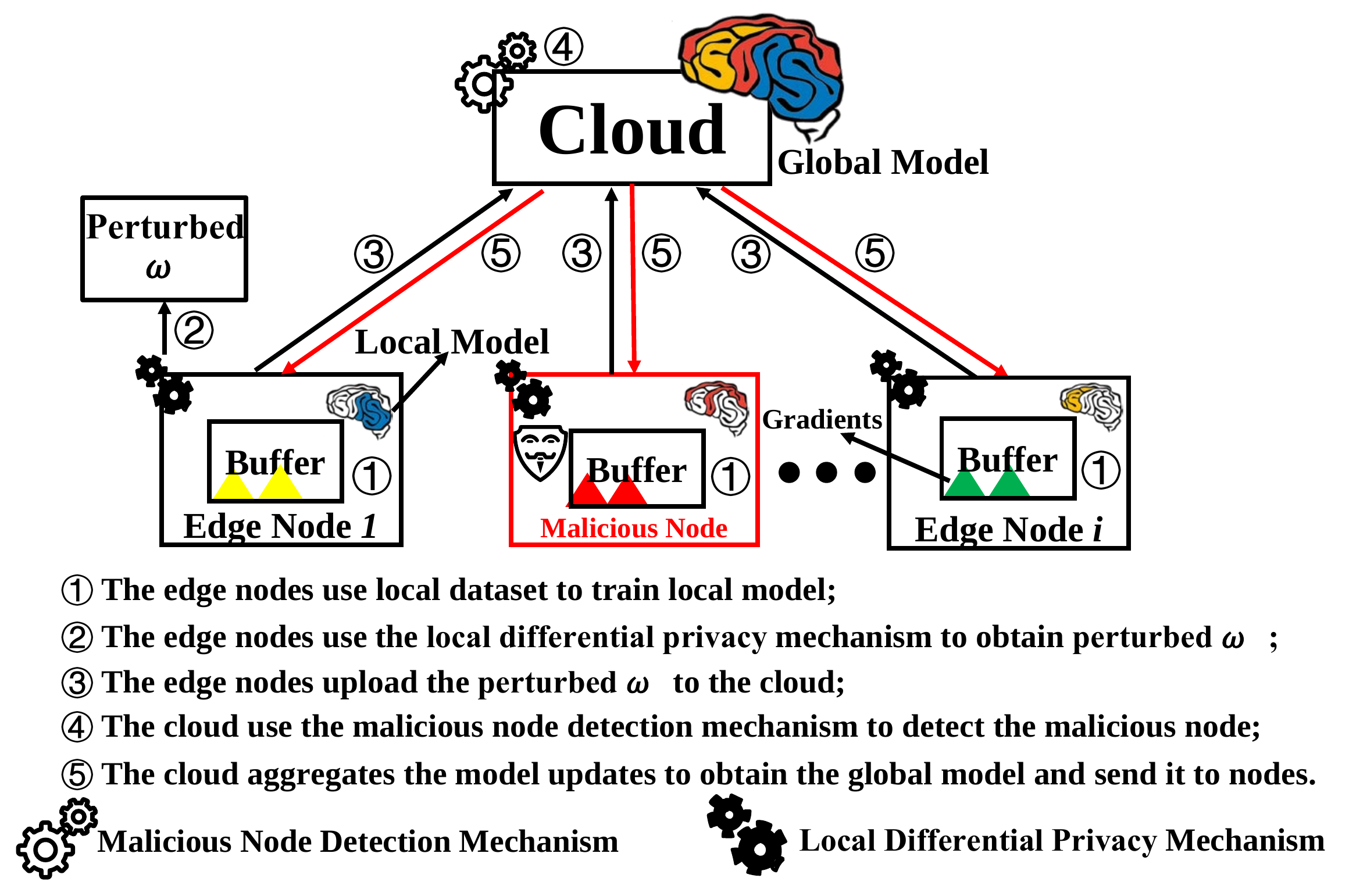}
	\caption{The overview of the {communication-efficient and attack-resistant} asynchronous federated edge learning framework.}
	\label{fig-1}
\end{figure}

In this section, considering {the above system design goal}, we propose a communication-efficient privacy-preserving asynchronous FEL framework, as illustrated in Fig. \ref{fig-1}. 

\subsection{Asynchronous Model Update Scheme}
It is challenging for edge nodes to obtain a better global model due to the huge communication overhead and network latency in IIoT. The reason is that the communication overhead and network delay will affect the gradient exchange between the {nodes} and the cloud, thus affecting the model aggregation at the cloud. To address this challenge, we propose the asynchronous model update scheme to update the global model. As shown in Fig. \ref{fig-10}, the scheme can reduce the {computation time} of the nodes by asynchronously aggregating the model updates, thereby improving the communication efficiency of {FEL}. The formal definition of the communication efficiency of the FEL framework is as follows:
\begin{equation}
	\kappa  = \frac{{Comm}}{{Comp + Comm}},
\end{equation}
where $\kappa$ represents communication efficiency, $Comp$ is the communication time, and $Comp$ is the computation time. It can be observed from Fig. \ref{fig-10} that the $Comp$ of asynchronous model update scheme is shorter than that of synchronous one, and hence the communication efficiency $\kappa$ of the asynchronous model update scheme is higher than that of synchronous one. In particular, in such a scheme, we prefer to upload gradients with large values. The reason is that a gradient with a large value contributes greatly to the update of the global model \cite{ref-6}. 

{If edge nodes upload gradients with large values first, the cloud will encounter a challenge: many gradients with small values are not aggregated, resulting in information loss in the global model.} To address this challenge, we propose a local gradient accumulation method in the asynchronous model update scheme, which can accumulate the small gradient {updates} in the gradient accumulation container. Furthermore, we also use local differential privacy mechanisms to perturb these gradients (see more details in Section \ref{ldp}).
\begin{figure}[!t]
	\centering
	\includegraphics[width=0.68\linewidth]{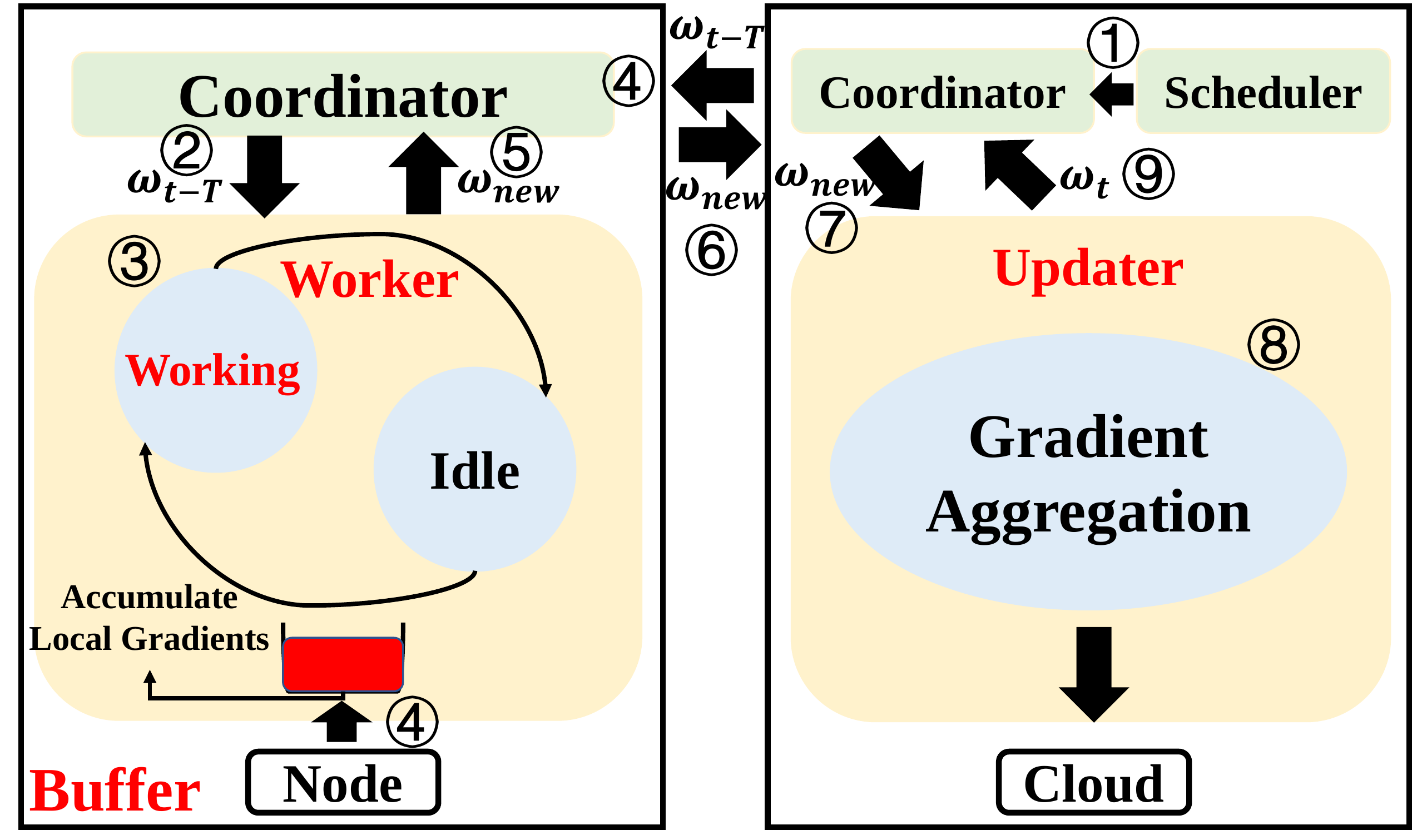}
	\caption{The overview of the asynchronous model update scheme.}
	\label{fig-3}
\end{figure}

As shown in Fig. \ref{fig-3}, the proposed scheme includes a scheduler, coordinator, worker, and updater. We assume that the proposed framework needs $\tau $ training epochs to achieve a federated global model optimization. In the $t$-th epoch, the cloud receives the locally trained model ${\omega _{new}}$ from all nodes and updates the global model through average aggregation. This scheme's workflow consists of 9 steps {in the following 3 phases:}
\begin{enumerate}[label=\roman*)]
	\item \textbf{\textit{Phase 1, Initialization:}} The scheduler triggers federated edge global model training tasks via the coordinator (step \textcircled{\scriptsize 1}). {The cloud} lets the coordinator send a delayed global model ${\omega _{t - T}}$ to each edge node (step \textcircled{\scriptsize 2}).
	\item \textbf{\textit{Phase 2, Training:}} The worker {in a node} receives the coordinator's command to {start the node-local training} (step \textcircled{\scriptsize 3}). Depending on the availability of the node, the worker process can switch between two states: working and idle. {The node performs local model training and continuously uploads gradients to the gradient accumulation container in the {local} buffer until the worker instructs to stop the training command (step \textcircled{\scriptsize 4}).} \textit{\textbf{We emphasize that the nodes do not need to wait for the cloud to finish model aggregation before starting the next round of training, but continue to perform local training until the worker instructs to stop the training instruction.}} Worker submits the local updated private gradient by using local differential privacy mechanism to the cloud via the coordinator (step \textcircled{\scriptsize 5}).
	\item \textbf{\textit{Phase 3, Aggregation:}} The scheduler queues the received models and feeds them back to the updater sequentially (steps \textcircled{\scriptsize 6}, \textcircled{\scriptsize 7}). The cloud uses 
	\begin{equation}
		{\omega _t} = \alpha {\omega _{t - 1}} + (1 - \alpha ){\omega _{new}},
	\end{equation}
	to perform gradient aggregation (step \textcircled{\scriptsize 8}), {where $\alpha  \in (0,1)$ and $\alpha$ is the mixing hyperparameter}. $\alpha$ can be used as an adjustment factor to control the impact of delayed updates on the global model, so we can optimize async global model optimization by adjusting $\alpha$. We find that the global model can be optimal when $\alpha = 0.5$ \cite{ref-19}. Finally, the cloud sends the optimal global model to all edge nodes via scheduler {for the next training iteration} (step \textcircled{\scriptsize 9}).
\end{enumerate}
In this scheme, steps \textcircled{\scriptsize 2} and \textcircled{\scriptsize 6} perform asynchronously in parallel so that the cloud can trigger {FEL} tasks on {the} edge nodes at any time. Likewise, the edge nodes can submit the local model updates to the cloud at any time. Such a scheme greatly reduces the {computation time} of the edge nodes, thereby improving the training efficiency.
\subsection{{Asynchronous} Local Differential Privacy Mechanism}\label{ldp}
{Unlike the {DP} mechanism, the {LDP} mechanism focuses on protecting the privacy of data owners during the data collection process. LDP is a {countermeasure} to provide strong privacy {guarantee}, which is commonly used in distributed machine learning. To protect the privacy of the updated gradients in asynchronous {FEL}, we add well-designed {noises} to {disturb the updated gradients by applying a Gaussian mechanism.}} We first define the Gaussian mechanism as follows:
\begin{definition}\label{defi-2}
\textbf{\textit{(Gaussian Mechanism \cite{ref-40}).}} \textit{Suppose an edge node wants to {generate} a function $f(x)$ of an input $x$ subject to $(\varepsilon ,\delta )$-LDP. Thus, we have:}
\begin{equation}
M(x) \buildrel \Delta \over = f(x) + \mathcal{N}(0,{\sigma ^2}{S^2}).
\end{equation}
\textit{{It is assumed} that $S=||f(x) - f(x')|{|_2} \le {\Delta _f}$, {(i.e., the sensitivity $S$ of the function $f(x)$ is bounded by ${\Delta _f}$)}, $\forall x,x'$, and then for any $\delta  \in (0,1)$, if and only if $\varepsilon  = \frac{{{\Delta _f}}}{\sigma }\sqrt {2\log \frac{{1.25}}{\delta }} $, Gaussian mechanism satisfies $(\varepsilon ,\delta )$-LDP.}
\end{definition}

\begin{definition}
\textbf{\textit{(Asynchronous Local Differential Privacy Mechanism (ALDP)).}} \textit{To achieve the local differential privacy guarantee for updated gradients of the edge nodes, edge nodes apply Gaussian mechanism {(see Definition (\ref{defi-2}))} on the updated gradients of each edge node by adding well-designed noise to perturb the gradients:
\begin{equation}\label{eq-8}
\resizebox{.9\hsize}{!}{$
\begin{aligned}
	{\omega _{t + 1}} = \alpha {\omega _t} + \underbrace {(1 - \alpha )}_{{\rm{ Factor}}}\frac{1}{K}(\underbrace {\overbrace {\sum\limits_{k = 1}^K {\Delta {\omega ^k}} /\max (1,\frac{{||\Delta {\omega ^k}|{|_2}}}{S})}^{{\rm{Sum\; of\; updates\; clipped\; at\; S }}} + \underbrace {\sum\limits_{k = 1}^K {\mathcal{N}(0,{\sigma ^2}{S^2})} }_{{\rm{Sum\; of\; noise\; at\; S}}}}_{{\rm{Gaussian\; mechanism\; approximating\; sum\; of\; updates}}}),
\end{aligned}$}
\end{equation}
where $K$ is the total number of edge nodes, each node's local model ${{\omega ^k}}$, $\forall k,||\Delta {\omega ^k}|{|_2} < S$.}
\end{definition}

{To} prevent {the} gradient explosion and information loss, we use gradient clipping \cite{ref-33,ref-54} to control gradient range when $\delta$ reaches a certain threshold and obtain the scaled versions $\Delta {{\bar \omega }^k}$, i.e., {$\Delta {{\bar \omega }^k}=\Delta {\omega ^k}/\max (1,\frac{{||\Delta {\omega ^k}|{|_2}}}{S})$}. We notice that {$\frac{{{\sigma ^{\rm{2}}}{S^{\rm{2}}}}}{K}$} controls the distortion of the average gradient update and $\sigma$, $D_k$ (i.e., the size of local dataset) defines the privacy loss caused when the random mechanism provides an average approximation. To track the privacy loss, we use the moments accountant {method} proposed in \cite{ref-40} to {evaluates $\delta $ given $\varepsilon$, $\sigma$, and $K$ where $\delta  \ll \frac{1}{K}$}. Compared with the standard composition theorem \cite{ref-33}, {the moments accountant} provides stricter restrictions on the privacy losses that occur. 

The {ALDP mechanism} is presented in Algorithm 1. The operating phases of ALDP mechanism {in the proposed FEL framework} are given as follows:
\begin{enumerate}[label=\roman*)]
	\item \textbf{\textit{Phase 1, Local Training:}} {Edge node receives the current global model $\omega_{t}$ and uses the local dataset to train the local model, i.e., $\arg \mathop {\min }\limits_{\omega  \in \mathbb{R}}{F_k}(\omega ),{F_k}(\omega ) = \frac{1}{{{D_k}}}\sum\nolimits_{i \in {D_k}} {{f_i}} (\omega )$. Each node determines the level of privacy protection (i.e., $\delta$) according to the size of its own dataset $D$ and privacy budget $\sigma$.}
	\item \textbf{\textit{Phase 2, Adding Noise:}} {Before adding noise to the updated gradients uploaded by edge nodes, the coordinator needs to use gradient clipping techniques to prevent gradient explosions. Then the coordinator uses Algorithm \ref{al-1} to add well-designed noise (refer to Eq. \eqref{eq-8}) to the gradients and upload {the} perturbed gradients to the cloud.}
	\item \textbf{\textit{Phase 3, Gradient Aggregation:}} {The cloud obtains the global model by aggregating perturbed gradients, i.e., ${\omega _t} = \alpha {\omega _t} + (1 - \alpha )\frac{1}{m}\bigg(\sum\limits_{k = 1}^m {((\Delta \omega _t^k} /\zeta ) + N(0,{\sigma ^2}{S^2}))\bigg)$ and sends new global model to each edge node.}
\end{enumerate}

\begin{algorithm}[!t]
\caption{Asynchronous local differential privacy mechanism.} \label{al-1}
\begin{algorithmic}[1] %这个1 表示每一行都显示数字
\REQUIRE ~~\\ %算法的输入参数：Input
Edge nodes $\mathcal{K} = \{ {k_1},{k_2},\ldots, k_i\} $, the local mini-batch size $B$,  the local dataset $D_k$,  the number of local epochs $E$, the learning rate $\eta$, the gradient optimization function $\nabla \mathcal{L}(\cdot ;\cdot)$;\\
\ENSURE ~~\\ %算法的输出：Output
The disturbed gradient $\omega$;\\
\STATE Initialize $ \omega_{t}$, $\mathrm{Accountant}$($\varepsilon$,$K$);
\FOR{round $t=1,2,\ldots,n$}
{
\STATE $ {\mathcal{K}_m} \leftarrow $ random select $m$ nodes from $\mathcal{K}$ {participates} in this round of training;
\STATE The cloud broadcasts the current global model $\omega_t$ to nodes in $ {\mathcal{K}_m}$;
\STATE $\delta  \leftarrow \mathrm{Accoutant}(D_m,{\sigma _m})$;
\WHILE{$\omega_{t}$ has not converged}
{
\FOR{each edge node $k \in {\mathcal{K}_m}$ \textbf{in parallel}}
{
\STATE Initialize $\omega _t^k = {\omega_t}$;
\STATE $\omega _{t + 1}^k \leftarrow$ $\mathrm{LocalUpdate}$ $(\omega _{t}^k,k,\varepsilon)$;
}
\ENDFOR
\IF{Gradient Clipping}
{
\STATE $\nabla \omega_t^k \leftarrow \mathrm{Local\_Gradient\_Clipping}\,(\nabla \omega_t^k)$;
}
\ENDIF
}
\STATE ${\omega _t} = \alpha {\omega _t} + (1 - \alpha )\frac{1}{m}\bigg(\sum\limits_{k = 1}^m {((\Delta \omega _t^k} /\zeta ) + N(0,{\sigma ^2}{S^2}))\bigg)$;
\STATE where $\zeta  = \max (1,\frac{{{\rm{||}}\Delta \omega _{t }^k|{|_2}}}{S})$;
\ENDWHILE
}
\ENDFOR
\STATE $\mathrm{LocalUpdate}$ $(\omega _{t}^k,k,\varepsilon)$:\hfill  $//$ Run on the local dataset;
\STATE $\mathcal{B} \leftarrow$ (split $\mathcal{S}_o$ into batches of size $B$);
\IF{each local epoch $i$ from 1 to $E$}
{\IF{batch $b\in \mathcal{B}$}
	{
	\STATE	$\omega  \leftarrow \omega  - \alpha  \cdot \nabla \mathcal{L}(\omega ;b)+ {\mathcal{N}(0,{\sigma ^2}{S^2})}$;
	}
\ENDIF
}
\ENDIF
\RETURN $\omega $.
\end{algorithmic}
\end{algorithm}

\subsection{Convergence Analysis}
\subsubsection{Assumptions} In this section, we will briefly introduce some necessary definitions and assumptions as follows:
\begin{definition}
\textit{\textbf{($\mathcal{L}$-smooth Function):} For $\forall x,y,$, if a differentiable function $f$ satisfies:
	\begin{equation}
		f(y) - f(x) \leqslant \langle \nabla f(x),y - x\rangle  + \frac{\mathcal{L}}{2}D,
	\end{equation}
where $\mathcal{L}>0$ and $D = ||y - x|{|^2}$, thus, we call the function $f$ is $\mathcal{L}$-function.}
\end{definition}

\begin{definition}
	\textit{\textbf{($\mu $-strongly Convex Function):} For $\forall x,y,$, if a differentiable function $f$ satisfies:
		\begin{equation}
			\langle \nabla f(x),y - x\rangle  + \frac{\mu }{2}D \leqslant f(y) - f(x),
		\end{equation}
		where $\mu \geqslant 0$, thus, we call the function $f$ is $\mu $-strongly convex function.}
\end{definition}

\begin{assumption}
	\textit{\textbf{(Existence of Global Optimum):} We assume that $\exists \chi  \in {\mathbb{R}^d}$ satisfies $\forall {\xi _*} \in \chi $ is a global optimum of $F(\xi)$, where $\nabla F(\xi ) = 0$ and $\xi _*  = {\inf _\xi }F(\xi )$.}
\end{assumption}
\subsubsection{Convergence Guarantees} Based on the above definitions and assumptions,, we use  
\begin{equation}
{\omega _t} = \alpha {\omega _t} + (1 - \alpha )\frac{1}{m}\bigg(\sum\limits_{k = 1}^m {((\Delta \omega _t^k} /\zeta ) + N(0,{\sigma ^2}{S^2}))\bigg),
\end{equation}
where $\alpha {\omega _{t}}$ represents the delayed model update, $(1 - \alpha )\frac{1}{m}\bigg(\sum\limits_{k = 1}^m {((\Delta \omega _t^k} /\zeta ) + N(0,{\sigma ^2}{S^2}))\bigg)$ represents the model update at the moment, and $\frac{1}{m}\sum\limits_{k = 1}^m {N(0,{\sigma ^2}{S^2})} $ denotes {the} sum of noise at $S$) to update the global model. In this {ALDP} mechanism, we use the hyperparameter $\alpha$ to control the weight of the delayed model update.
\begin{theorem}
Let $F$ be the global loss function where $F$ is {$\mathcal{L}$-smooth and $\mu $-strongly convex}. We assume that each node {performs} local model training at least $\nu \min $ epochs before submitting their perturbed gradients to the cloud. {For} $\forall x \in {{\mathbb {R}}^d},\forall z \sim {D_i}$,  {$i \in \mathbb{N}$}, we have {$\mathbb{E}\|\nabla f(x ; z)-\nabla F(x)\|^{2} \leq V_{1}$, and $\mathbb{E}\left[\|\nabla f(x ; z)\|^{2}\right] \leq V_{2}$ where $V_1$ and $V_2$ are constants}. Let $\lambda  < \frac{1}{\mathcal{L}}$, and ${\omega _{new}}$ and ${\xi _*} \in \chi $ denote {the} delayed model update and a global optimum at the moment after $T$ training epochs, respectively. Therefore, we have:
	\begin{equation}
		\begin{array}{*{20}{l}}
			{\mathbb{E}\left[ {F\left( {{\omega _T}} \right) - F\left( {{\xi _*}} \right)} \right] \leqslant } \\ 
			{{{[1 - \alpha  + \alpha (1 - \mu \lambda )]}^{T{\nu _{\min }}}}\left[ {F\left( {{\omega _0}} \right) - F\left( {{\xi _*}} \right)} \right]} \\ 
			{ + (1 - {{[1 - \alpha  + \alpha (1 - \mu \lambda )]}^{T{\nu _{\min }}}})\mathcal{O}({V_1} + {V_2}).} 
		\end{array}
	\end{equation}
	where $\mathcal{O}({V_1} + {V_2})$ is the additional error, ${(1 - \mu \lambda )^{T{\nu _{\min }}}}$ is the convergence rate, and $\alpha$ controls the trade-off between the convergence rate and {the} additional error caused by {the} variance.
\end{theorem}
\textbf{Remark:} \textit{When $\alpha  \to 1$, the convergence rate approaches ${(1 - \lambda \mu )^T}^{{\nu _{\min }}}$ and the variance is equal to $\mathcal{O}({V_1} + {V_2})$, i.e., 
\begin{equation}
	\mathbb{E}[F({x_T}) - F({\xi _*})] \leqslant {(1 - \lambda \mu )^{T{\nu _{\min }}}}[F({x_0}) - F({\xi _*})] + O({V_1} + {V_2}).
\end{equation}
When $\alpha  \to 0$, we have $(1 - {[1 - \alpha  + \alpha (1 - \mu \lambda )]^{T{\nu _{\min }}}}) \to 0$ which denotes that  the variance is {also} reduced to $0$. Therefore, as {the iteration number} of the asynchronous model update scheme increases, the variance of {errors} gradually decreases to $0$, which means that this scheme can reach convergence. Detailed proofs can be found in the appendix.}

\subsection{Malicious Node Detection Mechanism}

\begin{figure}[!t]
	\centering
	\includegraphics[width=1\linewidth]{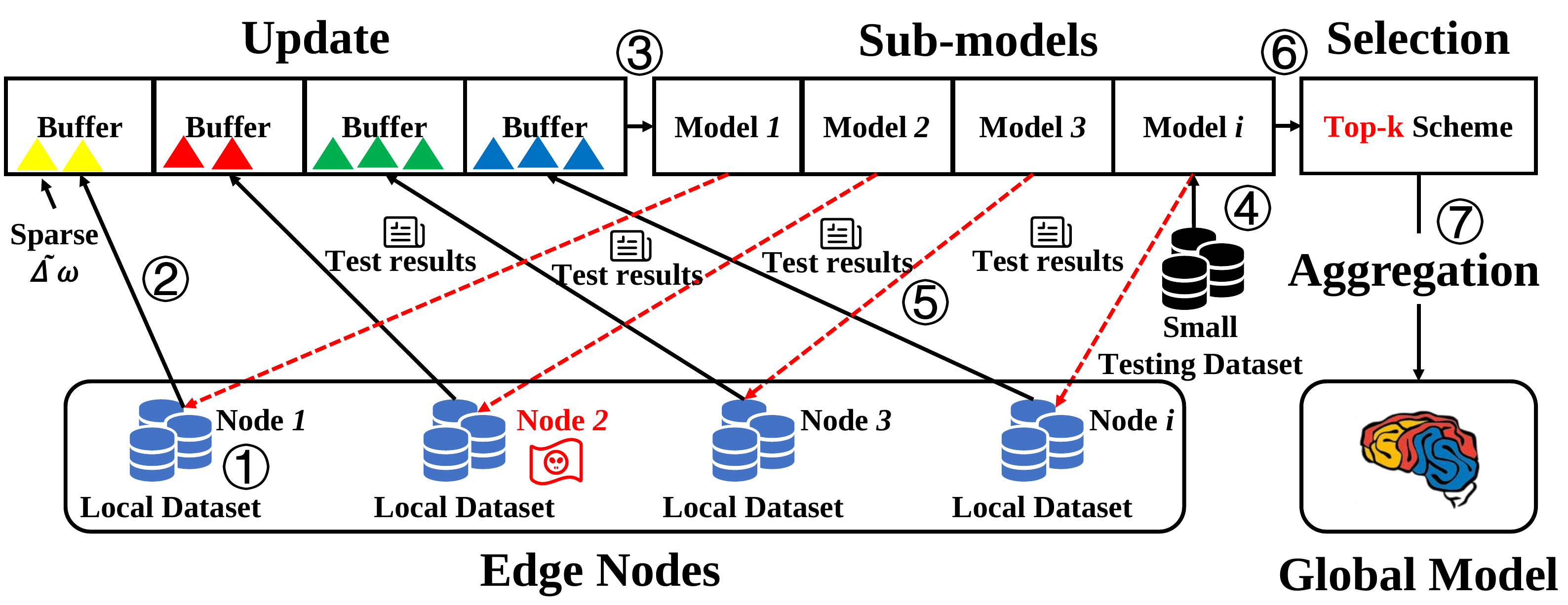}
	\caption{The overview of the malicious node detection scheme.}
	\label{fig-4}
\end{figure}
In this section, to detect malicious edge nodes in the proposed framework, we propose a malicious node detection mechanism, as illustrated in Fig. \ref{fig-4}. Specifically, the cloud uses a specific testing dataset to test the sub-models uploaded by edge nodes to detect malicious nodes. The cloud can create a testing dataset in the cloud when performing global model learning tasks \cite{ref-2,ref-12,ref-22}. {For example, Nguyen \textit{et al.} in \cite{ref-22} proposed an anomaly detection system by creating a testing dataset in the cloud to detect the anomalies for IIoT. Therefore, our solution is practical and feasible in the real world.}

The proposed {detection mechanism} consists of 7 steps, {as shown in Fig. \ref{fig-4}.} Edge nodes use the local dataset to train the {sub-models} (step \textcircled{\scriptsize 1}). The node uploads the perturbed gradients to the buffer (step \textcircled{\scriptsize 2}). Note that the perturbed gradient is generated by {the proposed ALDP} mechanism. The coordinator generates corresponding sub-models based on the order of the nodes (step \textcircled{\scriptsize 3}). The coordinator tests each sub-model using a testing dataset created by the cloud (step \textcircled{\scriptsize 4}). The coordinator sends the test results to each node and marks the malicious node (step \textcircled{\scriptsize 5}). Specifically, the coordinator {decides} normal nodes according to the malicious detection algorithm (step \textcircled{\scriptsize 6}). The pseudo-code of the proposed {malicious detection} algorithm is presented in Algorithm 2. The cloud aggregates the model updates uploaded by all normal nodes to obtain the global model (as shown in step \textcircled{\scriptsize 7}).
As shown in Algorithm 2, the cloud can set the accuracy threshold {${Thr}$} according to the requirements of the task. {The} threshold can be also flexibly adjusted for different tasks.

\begin{algorithm}[!t]\label{al-2}
\caption{{Malicious node detection mechanism.}} 
\begin{algorithmic}[1] %这个1 表示每一行都显示数字
\REQUIRE ~~\\ %算法的输入参数：Input
Edge nodes $\mathcal{K} = \{ {k_1},{k_2},\ldots,{k_i}\} $, their local model $\mathrm{M} = \{ {M_1},{M_2},\ldots,{M_k}\}$, a constant $s$ given by the cloud;\\
\ENSURE ~~\\ %算法的输出：Output
Normal edge nodes $\mathcal{K}_{norm} \subseteq  \mathcal{K}$ and parameter $\omega$;\\
Initialize parameter $ \omega_{t}$;\\
\FOR{round $t = 1,2,\ldots,n$}
{
		{\STATE Follows  line (3)--(9) of Algorithm 1;}
		\FOR{each Sub-model $M_k \in \mathrm{M}$ \textbf{in parallel}}
		{
			\STATE Each sub-model $M_k$ runs on the testing dataset to obtain the accuracy $A_k$;
			%$\mathcal{A} \leftarrow $ Accuracy of each sub-model run on the testing dataset\;
			{\STATE Put each  accuracy $A_k$ into the accuracy set $\mathcal{A}$;}
		}
		\ENDFOR
	\STATE	${{Thr}} \leftarrow \mathrm{Top}\,s\%\ \mathrm{of}\ \mathcal{A}$;
		\FOR{$j = 0,1, \ldots ,k$}
		{
			\IF{$A_j \in \mathcal{A}$ $ \wedge $ ${A_j} > {Thr}$}
			{
			\STATE	The coordinator marks the {$j$-th} node as a normal node;
			\STATE	Put this node $k_j$ to the normal node set {$\mathcal{K}_{norm}$};
			}
			\ELSE{
			\STATE	The coordinator marks {$j$-th} node as a malicious node;
			}  	
			\ENDIF
		}
		\ENDFOR
	\STATE	${\omega _{t + 1}} \leftarrow \frac{1}{{|\{\mathcal{K}_{norm}\}|}}\sum\nolimits_{k \in {\mathcal{K}_{norm}}} {\omega _{t + 1}^k}$;
	}
	\ENDFOR
	\RETURN {$\mathcal{K}_{norm}$ and $\omega $}.
\end{algorithmic}
\end{algorithm}

\subsection{{Security Analysis}}
In this section, we provide a comprehensive security analysis of the mitigation measures proposed for these malicious attacks.
Reference \cite{ref-24} shows that the cloud {could acquire} the private data of nodes by {utilizing} the updated gradient information. For example, a malicious cloud can launch a gradient leakage attack by utilizing gradients uploaded by nodes to infer the node's {private} data \cite{ref-12}. The reason is that the updated gradient contains the distribution of local training data, and {an} adversary can use this distribution to restore private training data through reverse engineering.
Therefore, we {introduce} an LDP mechanism to mitigate gradient leakage attacks at the cloud-side. Specifically, {the nodes} only send perturbed gradients by using the {ALDP} mechanism to the cloud and the {perturbed gradients do} not contain the entire distribution of local training data.
{Furthermore, it is difficult for the adversary to collect all the gradient information of a specific node that uses an asynchronous gradient update method.}

The edge nodes implement label-flipping attacks by modifying the labels of the local training dataset, which will generate poisoning model updates and cause the model accuracy to decrease. Therefore, we can use {the accuracy performance} as a metric to measure the quality of local model updates {to mitigate label-flipping attacks.}
We use the accuracy of the top $s\%$ nodes as the baseline. The nodes with an accuracy higher than the baseline are marked as normal nodes, {otherwise they are marked as malicious nodes}. The cloud can choose different hyperparameters $s$ according to different learning tasks, but {$s\%$} is generally greater than {0.5} \cite{ref-16}. The portion of malicious nodes indicates the {intensity} of label-flipping attacks. So we investigate the impact of the proportion of malicious nodes on the performance of the proposed framework. Experimental results show that the proposed framework is relatively robust against label-flipping attacks.

\section{{Experimental Results}}\label{sec-6}
\subsection{Experiment Setup}
In this section, {the proposed framework  is conducted on MNIST and CIFAR-10 datasets for the performance analysis.} 
{All simulations are implemented on the same computing environment (Linux Ubuntu 18.04, Intel i5-4210M CPU, 16GB RAM, and 512GB SSD) with Pytorch and PySyft \cite{ref-25}}.

\textbf{Model:} In this experiment, to build a model that can be easily deployed in edge nodes, we {use} a simple {deep learning} model (i.e., {CNN} with 2 convolutional layers followed by 1 fully connected layer) for classification tasks on the MNIST and CIFAR-10 datasets. 

\textbf{Datasets:} MNIST is a handwritten digital image dataset, which contains 60,000 training samples and 10,000 test samples. Each picture consists of $28 \times 28$ pixels, each pixel is represented by a gray value, and the label is one-hot code: 0-9. The CIFAR-10 dataset consists of 10 types of $32\times 32$ color pictures, a total of 60,000 pictures, each of which contains 6,000 pictures. Among them, 50,000 pictures are used as the training dataset and 10,000 pictures are used as the test dataset. The pixel {values of images} in all datasets are normalized into [0,1]. 

{\textbf{Parameters:} During the {simulations}, the number of edge nodes $K = 10$, {including} 3 malicious edge nodes, the learning rate $\eta = 0.001$, the training {epochs} $E= 1000$, and the mini-batch size $B = 128$.}

{\textbf{Attacks:} Malicious edge nodes {carry out} 100 times label-flipping attacks by changing all labels {`1' to `7'} in the MNIST dataset and changing all classes {`dog' to `cat'} in the  CIFAR-10 dataset. Furthermore, we assume that the cloud is malicious in each {simulation}. To evaluate the performance of the malicious node detection mechanism, we define the attack success rate (ASR) as follows:
\begin{definition}
\textbf{\textit{Attack Success Rate (ASR) \cite{ref-68,ref-69}}} is the percentage of successfully reconstructed training data over the number of training data being attacked.
\end{definition}}

\subsection{Hyperparameter Selection of the Malicious Node Detection Mechanism}
In the context of the malicious node detection mechanism, {the} proper hyperparameter selection, i.e., a threshold of the accuracy baseline of {sub-models} generated by the normal edge nodes, is a notable factor {to determine} the proposed mechanism performance. { {We} investigate the performance of the proposed mechanism with different thresholds and try to find a {optimal} threshold. In particular, we employ {$s \in \{50,60,70,80,90\}$} to adjust the best threshold of the proposed mechanism \cite{ref-46,ref-54}.} {We use {ASR} to indicate the performance of the proposed mechanism and use MNIST and CIFAR-10 datasets to evaluate the proposed mechanism accuracy with the selected threshold.} 

As shown in Fig. \ref{a-1}, {the larger $s$ is, the lower ASR achieves. It implies that the better resistance of the proposed mechanism to malicious attacks with the larger threshold.} Fig. \ref{b-1} shows that the accuracy performances of {the} proposed mechanism with different thresholds. {It is shown} that the highest accuracy {can be achieved} when $s = 80$ {on the two datasets}. {The reason is that when the threshold becomes larger, the more nodes including normal nodes will be filtered out to decrease the accuracy performance.} {Therefore, we choose $s = 80$ as the threshold for the following experiments  to make a  trade-off between ASR and accuracy.}

\begin{figure}[t]
	\centering
	\large
	\subfigure []{\includegraphics[width=0.45\linewidth]{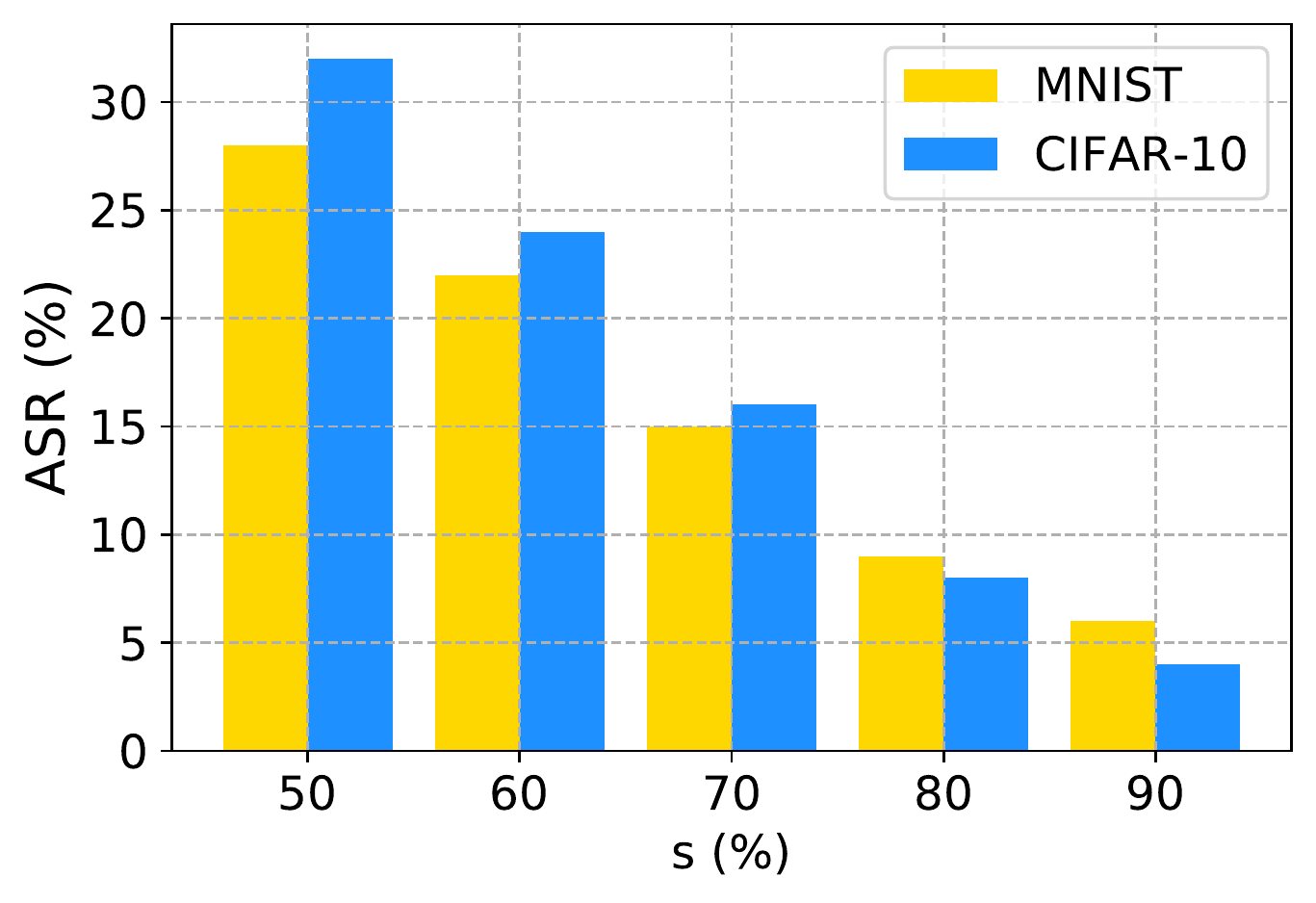}
		\label{a-1}}
	\hfill
	\subfigure[]{	\includegraphics[width=0.45\linewidth]{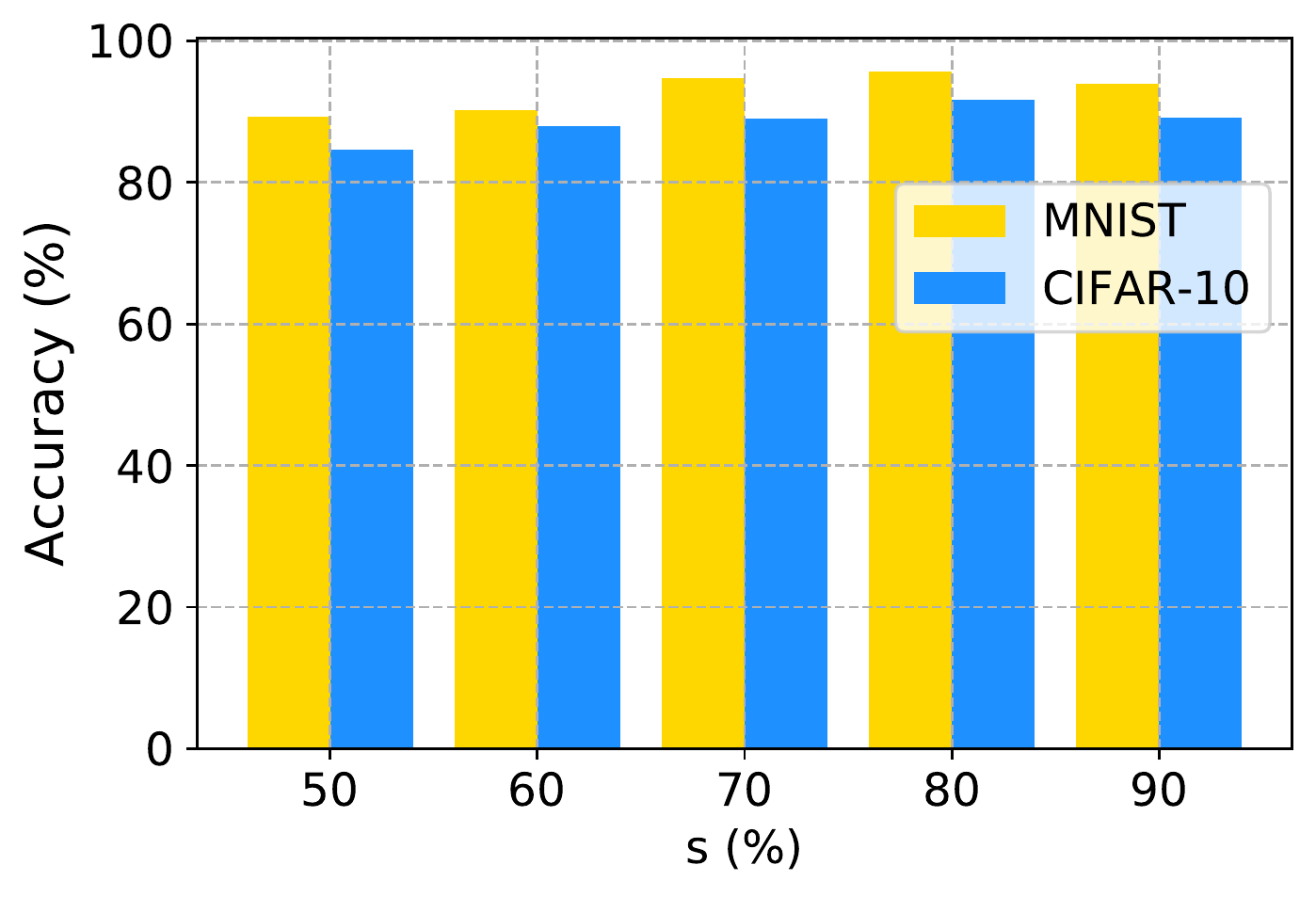}
		\label{b-1}}
	\caption{{The performance} of the proposed malicious node detection mechanism: (a) ASR; (b) Accuracy.}
	\label{fig-5}
\end{figure}
\begin{figure}[t]
	\centering
	\large
	\subfigure []{\includegraphics[width=0.45\linewidth]{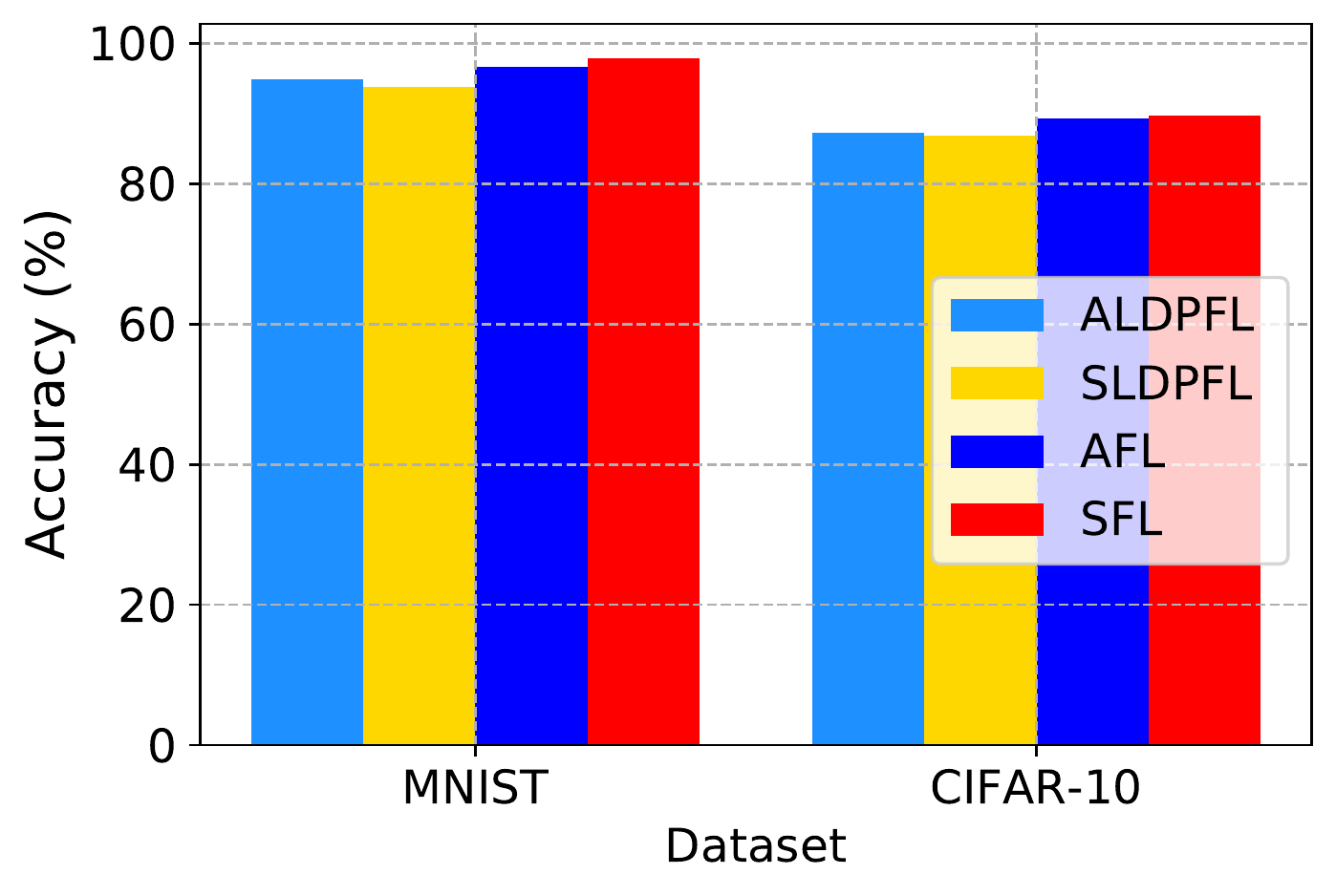}
		\label{a-2}}
	\hfill
	\subfigure[]{	\includegraphics[width=0.45\linewidth]{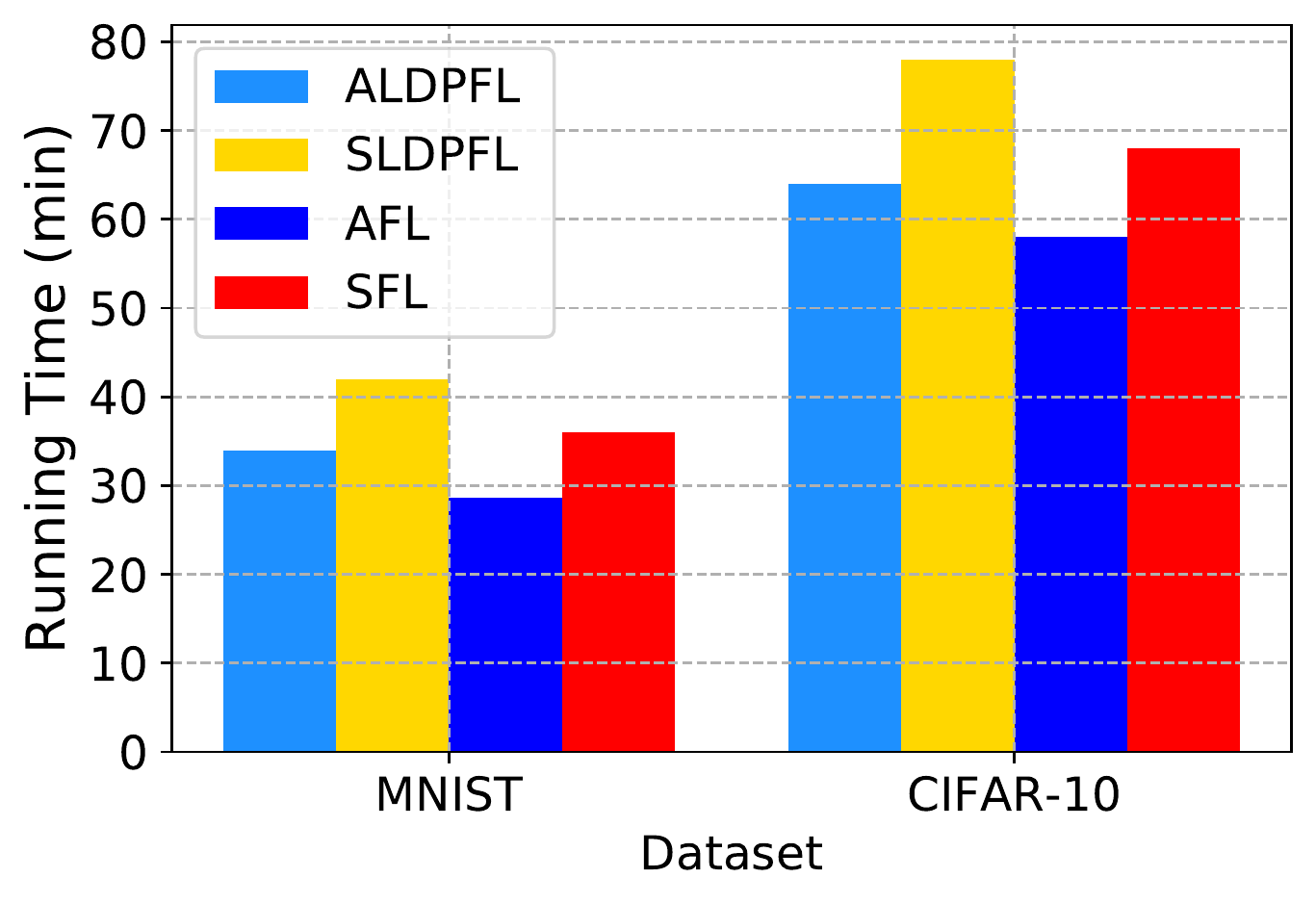}
		\label{b-2}}
	\caption{{The performance} comparison between the proposed model and SLDPFL, AFL, and SFL models: (a) Accuracy; (b) Running Time.}
	\label{fig-6}
\end{figure}

\begin{figure}[t]
	\centering
	\large
	\subfigure []{\includegraphics[width=0.45\linewidth]{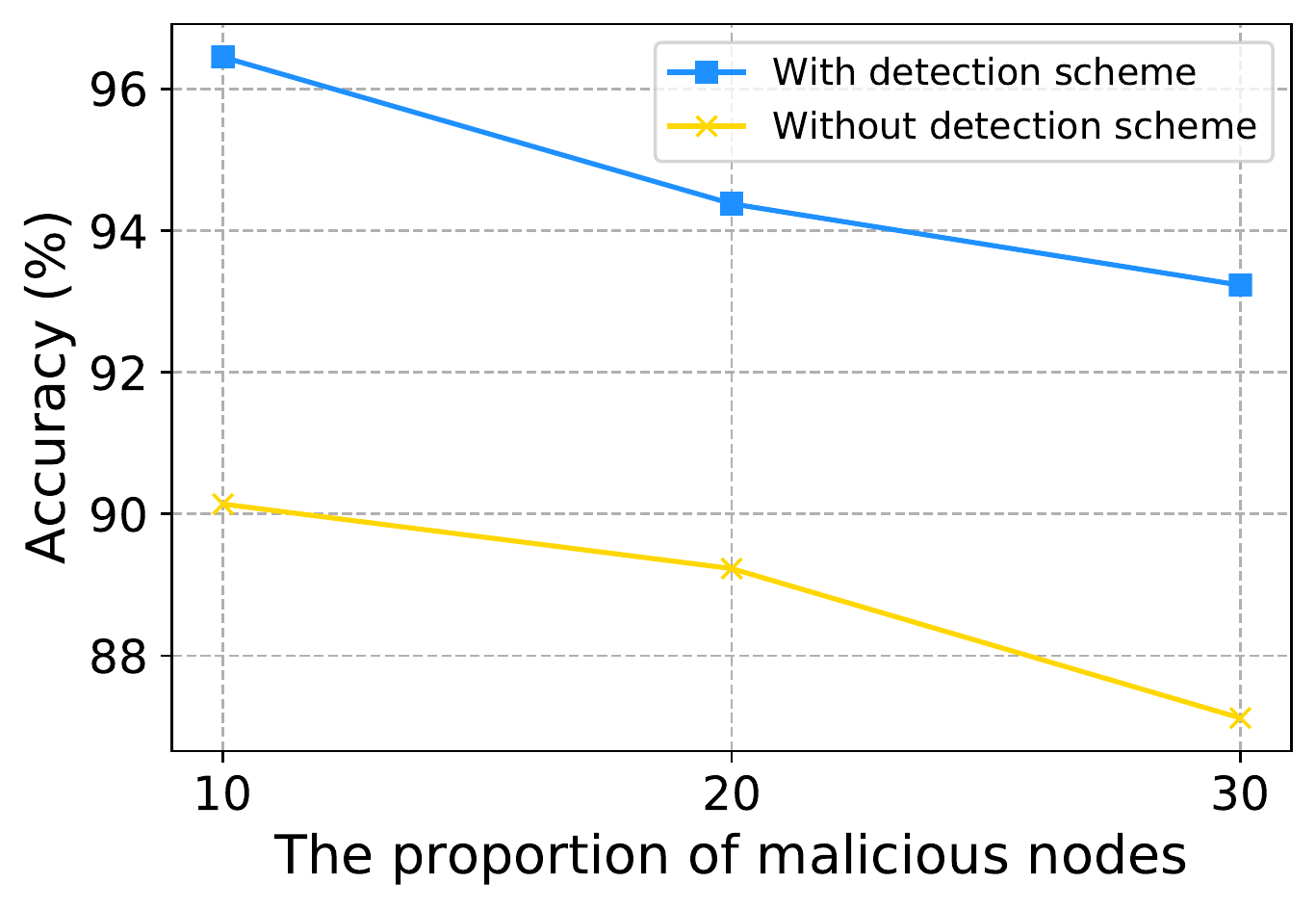}
		\label{a-3}}
	\hfill
	\subfigure[]{	\includegraphics[width=0.45\linewidth]{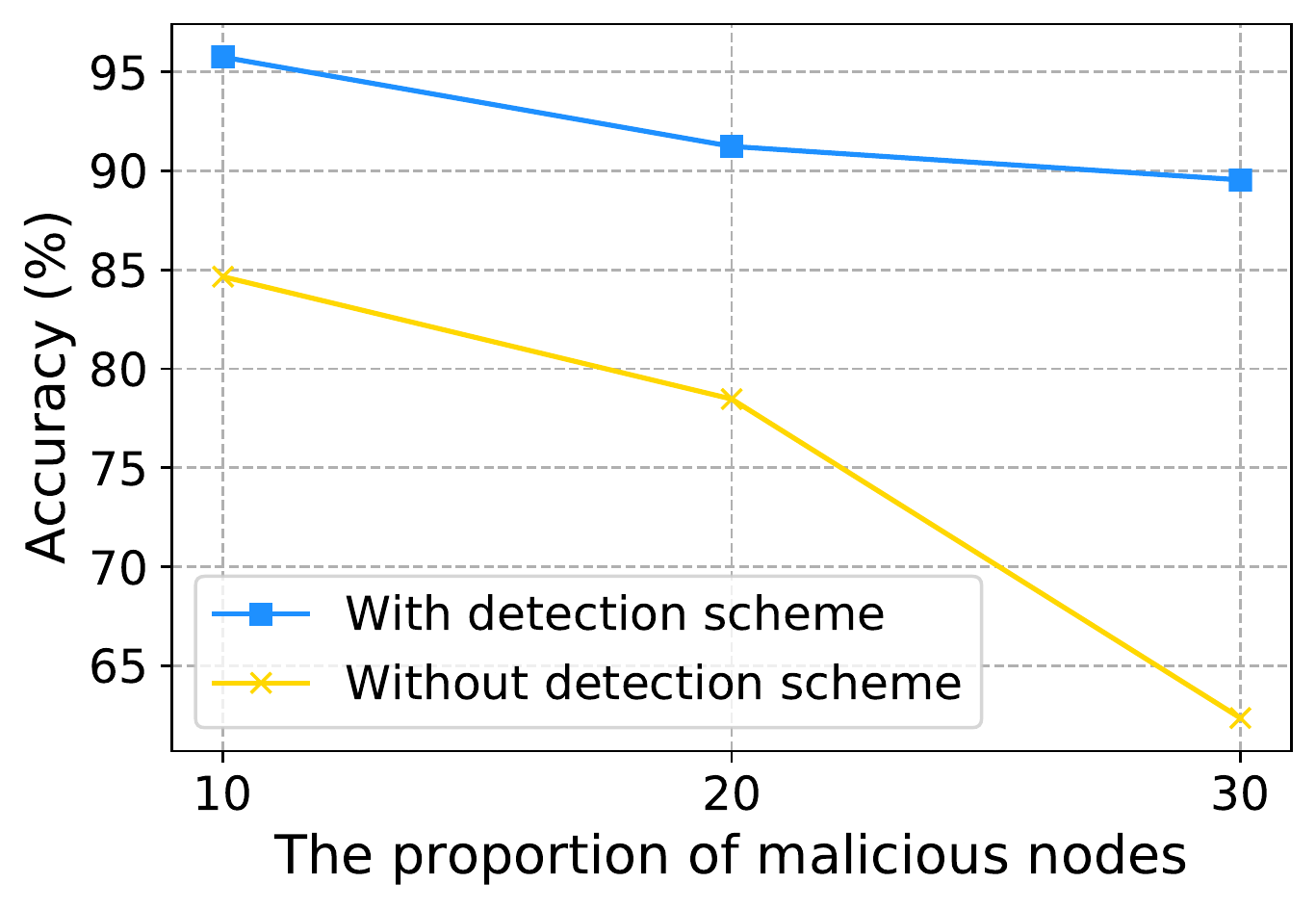}
		\label{b-3}}
	\caption{{The performance} comparison between ALDPFL model with detection scheme and without detection scheme one on MNIST dataset: (a) General task; (b) Special task.}
	\label{fig-7}
\end{figure}
\subsection{{Performance of the Proposed Framework}}
We {further compare} the performance of the asynchronous local differentially private FEL (ALDPFL) framework {with} that of synchronous local differentially private FEL (SLDPFL) \cite{ref-16}, asynchronous FEL (AFL) \cite{ref-19}, and synchronous FEL (SFL) \cite{ref-25} methods with an identical simulation {settings}. In accordance to \cite{ref-40}, we fixed $\varepsilon = 8, \delta = 10^{-3}$, and {the communication cost for each communication round}. Fig. \ref{a-2} shows the performance comparison {of} ALDPFL, SLDPFL, AFL, and SFL {models} for classification tasks on MNIST and CIFAR-10 dataset. We can find that the accuracy of ALDPFL model is close to those of the AFL {or} SFL models. {It means} that ALDPFL sacrifices little accuracy {but} {provides the} node-level privacy protection with {the LDP mechanism.} {Fig. \ref{b-2} shows that the running time of ALDPFL is close to that of SFL. In addition, it is longer than that of AFL but lower than that of SLDPFL.}
{This is because an asynchronous model update method can improve communication efficiency by reducing nodes' computation time.} Therefore, ALDPFL model not only achieves {the} accuracy and communication efficiency comparable to {those} of AFL {model}, but also provides a strong privacy guarantee.

\subsection{{Performance of Mitigating Label-flipping Attacks}}\label{ex-3}
{We} assume that a proportion $p \in \{10\%,20\%,30\%\}$ of malicious {edge} nodes exist in the edge nodes. Malicious edge nodes implement label-flipping attacks by changing all labels `1' to `7' in {their} local MNIST datasets. In this experiment, we consider two {accuracy evaluation} tasks: the  {general} task and the {special task}. {For the general task,} we compare the accuracy performances of ALDPFL model {with/without} the detection scheme on MNIST dataset. In Fig. \ref{a-3}, {it can be observed that the robustness of the proposed framework becomes more evident than the conventional one without detection as the proportion $p$ of malicious nodes increases.} {This suggests} that the proposed framework uses the malicious node detection mechanism to detect malicious nodes, thereby preventing malicious nodes from participating in {the global model training.} {For the special task, we compare the classification accuracy of ALDPFL model in the cases of with/without the detection scheme on a given digit in MNIST dataset, e.g., the digit `1'. Fig. \ref{b-3} illustrates that the accuracy of ALDPFL framework without the detection mechanism decreases more rapidly than that with the detection mechanism. In contrast, the proposed framework is robust for the special task despite the increasing proportion $p$ of malicious nodes. In summary, the proposed framework can use the malicious node mechanism to mitigate label-flipping attacks.}

\section{Conclusion}\label{sec-7}
In this paper, we propose a communication-efficient and secure asynchronous FEL framework for edge computing in IIoT. We present an asynchronous model update method to improve communication efficiency in the proposed framework. In this method, edge nodes can upload gradients at any time without waiting for other nodes, and the cloud can aggregate gradients at any time, so it can improve communication efficiency by reducing nodes' computation time. ALDP mechanism is introduced to mitigate gradient leakage attacks where the malicious cloud in IIoT can not restore the raw gradient distribution from the perturbed gradients. We also mathematically prove that ALDPM can achieve convergence. Additionally, a malicious node detection mechanism is designed to detect malicious nodes in the proposed framework to mitigate label-flipping attacks. {The cloud-only needs to use a testing dataset to test the quality of each local model to achieve malicious node detection. Experimental results show that the proposed framework can provide node-level privacy protection and negligibly impact the accuracy of the global model.}

{In the future, our work will mainly focus on the following two aspects: (1) Communication efficiency: To combine communication-efficient technology with cryptography, we plan to develop a symmetric gradient quantization technology to improve the communication efficiency of the system and meet the requirements of cryptography. (2) Attack resistance: To deal with complex attack scenarios, we plan to design an anomaly detection algorithm based on neural networks to improve the robustness of the detection algorithm.}

%\ifCLASSOPTIONcaptionsoff
%\newpage
%\fi

%\begin{acks}
%To Robert, for the bagels and explaining CMYK and color spaces.
%\end{acks}

\bibliographystyle{ACM-Reference-Format}
\bibliography{reference}

\appendix

\section{Proofs}

\subsection{Proof of Theorems}
\begin{theorem}
Let $F$ be the global loss function where $F$ is {$\mathcal{L}$-smooth and $\mu $-strongly convex}. We assume that each node {performs} local model training at least $\nu \min $ epochs before submitting their perturbed gradients to the cloud. {For} $\forall x \in {{\mathbb {R}}^d},\forall z \sim {D_i}$,  {$i \in \mathbb{N}$}, we have {$\mathbb{E}\|\nabla f(x ; z)-\nabla F(x)\|^{2} \leq V_{1}$, and $\mathbb{E}\left[\|\nabla f(x ; z)\|^{2}\right] \leq V_{2}$ where $V_1$ and $V_2$ are constants}. Let $\lambda  < \frac{1}{\mathcal{L}}$, and ${\omega _{new}}$ and ${\xi _*} \in \chi $ denote {the} delayed model update and a global optimum at the moment after $T$ training epochs, respectively. Therefore, we have:
\begin{equation}
		\begin{array}{*{20}{l}}
			{\mathbb{E}\left[ {F\left( {{\omega _T}} \right) - F\left( {{\xi _*}} \right)} \right] \leqslant } \\ 
			{{{[1 - \alpha  + \alpha (1 - \mu \lambda )]}^{T{\nu _{\min }}}}\left[ {F\left( {{\omega _0}} \right) - F\left( {{\xi _*}} \right)} \right]} \\ 
			{ + (1 - {{[1 - \alpha  + \alpha (1 - \mu \lambda )]}^{T{\nu _{\min }}}})\mathcal{O}({V_1} + {V_2}).} 
		\end{array}
\end{equation}
where $\mathcal{O}({V_1} + {V_2})$ is the additional error, ${(1 - \mu \lambda )^{T{\nu _{\min }}}}$ is the convergence rate, and $\alpha$ controls the trade-off between the convergence rate and {the} additional error caused by {the} variance.
\end{theorem}

\begin{proof}
	On the cloud side, we have ${x_{t + 1}} = \alpha {x_t} + (1 - \alpha )\frac{1}{m}\bigg(\sum\limits_{k = 1}^m {((\Delta x_t^k} /\zeta ) + N(0,{\sigma ^2}{S^2}))\bigg)$. Thus, we have 
	\begin{equation}
		\begin{gathered}
			\mathbb{E}[F({x_{t + 1}}) - F({\xi _*})] \hfill \\
			\leqslant (1 - \alpha )[F({x_t}) - F({\xi _*})] + \alpha \mathbb{E}[F({x_\tau }) - F({\xi _*})] \hfill \\
			\leqslant (1 - \alpha  + \alpha {(1 - \lambda \mu )^{{\nu _{\min }}}})[F({x_t}) - F({\xi _*})] + \frac{{\alpha {\nu _{\max }}\lambda {V_1}}}{2} + \frac{\alpha }{{2\mu }}{(1 - \lambda \mu )^{{\nu _{\max }}}}{V_2} \hfill \\
			\leqslant (1 - \alpha  + \alpha {(1 - \lambda \mu )^{{\nu _{\min }}}})[F({x_t}) - F({\xi _*})] + \frac{{\alpha ({V_1} + {V_2})}}{{2\mu }}. \hfill \\ 
		\end{gathered} 
	\end{equation}
\end{proof}

\end{document}